\documentclass{article}

\usepackage{microtype}
\usepackage{graphicx}
\usepackage{subfigure}
\usepackage{booktabs} 

\usepackage{algorithm} 
\usepackage[noend]{algpseudocode}
\renewcommand\algorithmicthen{}
\renewcommand\algorithmicdo{}
\renewcommand{\algorithmicrequire}{\textbf{Input: }}
\renewcommand{\algorithmicensure}{\textbf{Output: }}

\makeatletter
\def\BState{\State\hskip-\ALG@thistlm}
\makeatother

\usepackage{hyperref}

      \usepackage[final]{neurips_2019}

\usepackage{amsmath,bm,graphicx,psfrag,epsf,enumerate,placeins,float,sectsty,amstext,relsize,wrapfig,listings,subfigure,amsfonts,amsthm,mathabx,enumitem,hyperref,grffile,colortbl,dsfont,bbm, float}

\usepackage[percent]{overpic}

\floatstyle{plain}
\newfloat{algs}{tb} 

\usepackage{tabularx}

\usepackage{multibib} 
\newcites{si}{Additional References for the Supplementary Material}

\newcommand{\nocontentsline}[3]{}
\let\oldaddcontentsline\addcontentsline
\newcommand{\tocless}[2]{%
  \let\addcontentsline\nocontentsline
  #1{#2}
  \let\addcontentsline\oldaddcontentsline}

\renewcommand{\cite}{\citep} 

\newcommand{\beginsupplement}{ 
        \setcounter{section}{0}
        \renewcommand{\thesection}{S\arabic{section}} %
         \renewcommand{\thesubsection}{\thesection.\arabic{subsection}}
        \setcounter{table}{0}
        \renewcommand{\thetable}{S\arabic{table}} %
        \setcounter{figure}{0}
        \renewcommand{\thefigure}{S\arabic{figure}} %
     }

\DeclareMathOperator*{\argmax}{argmax}
\DeclareMathOperator*{\argmin}{argmin}

\newcommand{\smallup}{\mathrel{\text{{\scalebox{0.8}{\hspace*{-2pt} $\uparrow$}}}}}
\newcommand{\smalldown}{\mathrel{\text{{\scalebox{0.8}{\hspace*{-2pt} $\downarrow$}}}}}

\definecolor{LightGray}{gray}{0.9}
\newcommand\numberthis{\addtocounter{equation}{1}\tag{\theequation}}
\newcommand{\todo}[1][]{} 

\newtheorem{thm}{Theorem}
\newtheorem{lem}{Lemma}

\newcounter{factnum}
\newcounter{claimnum}
 \newcounter{defnum}
\newcounter{assumptioncount} 

\newtheorem{innercustomthm}{Theorem} 
\newenvironment{customthm}[1]
  {\renewcommand\theinnercustomthm{#1}\innercustomthm}
  {\endinnercustomthm}

\newtheorem{innercustomlem}{Lemma} 
\newenvironment{customlem}[1]
  {\renewcommand\theinnercustomlem{#1}\innercustomlem}
  {\endinnercustomlem}

\allowdisplaybreaks  

\renewcommand{\algorithmicrequire}{\textbf{Input:}} 
\renewcommand{\algorithmicensure}{\textbf{Output:}}
\newcommand\NoDo{\renewcommand\algorithmicdo{}} 

\newcommand{\algfirstabbrev}{\mbox{OPOK}}
\newcommand{\algsecondabbrev}{\mbox{OPOL}}
\newcommand{\flaxmanabbrev}{\mbox{GDG}}

\title{Low-Rank Bandit Methods for High-Dimensional Dynamic Pricing}

\author{%
  Jonas Mueller  \\
  MIT CSAIL \\
   \texttt{jonasmueller@csail.mit.edu} \\
   \And
   Vasilis Syrgkanis \\
   Microsoft Research \\
   \texttt{vasy@microsoft.com} 
   \And
   Matt Taddy \\
   Chicago Booth \\
   \texttt{taddy@chicagobooth.edu} 
}

\begin{document}

\maketitle

\begin{abstract}

We consider dynamic pricing with many products under an evolving but low-dimensional demand model.   Assuming the temporal variation in cross-elasticities exhibits low-rank structure based on fixed (latent) features of the products, we show that the revenue maximization problem  reduces to an online bandit convex optimization with side information given by the observed demands.  We design dynamic pricing algorithms whose revenue approaches that of the best fixed price vector in hindsight, at a rate that only depends on the intrinsic rank of the demand model and not the number of products.  Our approach applies a bandit convex optimization algorithm in a projected low-dimensional space spanned by the latent product features, while simultaneously learning this span via online singular value decomposition of a carefully-crafted matrix containing the observed demands. 

\end{abstract}

\tocless\section{Introduction}
\label{sec:intro}

In this work, we consider a seller offering $N$ products, where $N$ is large, and the pricing of certain products may influence the demand for others in unknown ways. We let $\mathbf{p}_t \in \mathbb{R}^N$ denote the vector of selected prices at which each product is sold during time period $t \in \{1,\dots, T\}$, which results in total demands for the products over this period represented in the vector $\mathbf{q}_t \in \mathbb{R}^N$.  Note that $\mathbf{q}_t$ represents a (noisy) evaluation of the aggregate demand curve at the chosen prices $\mathbf{p}_t$, but we never observe the counterfactual demand that would have resulted had we selected a different price-point.  This is referred to as \emph{bandit} feedback in the online optimization literature \cite{Dani07}. 
Our goal is find a setting of the prices for each time period to maximize the \emph{total revenue} of the seller (over all rounds).  This is equivalent to minimizing the negative revenue over time:
\begin{equation*}
R(\mathbf{p}_1,\dots, \mathbf{p}_T) = \sum_{t=1}^T R_t(\mathbf{p}_t) \ \text{ where } R_t(\mathbf{p}_t) = - \langle \mathbf{q}_t, \mathbf{p}_t \rangle
\end{equation*}
We can alternatively maximize total profits instead of revenue by simply redefining $\mathbf{p}_t$ as the difference between the product-prices and the cost of each product-unit. 
In practice, the seller can only consider prices within some constraint set  $\mathcal{S} \subset \mathbb{R}^N$, which we assume is convex throughout.
To find the optimal prices, we introduce the following linear model of the aggregate demands, which is allowed to change over time in a nonstationary fashion:
\begin{equation}
\mathbf{q}_t = \mathbf{c}_t - \mathbf{B}_t \mathbf{p}_t + \bm{\epsilon}_t  
\label{eq:basemodel}
\end{equation}

Here, $\mathbf{c}_t \in \mathbb{R}^N$ denotes the baseline demand for each product in round  $t$.  $\mathbf{B}_t \in \mathbb{R}^{N \times N}$ is an \emph{asymmetric} matrix of demand elasticities which represents how changing the price of one product may affect the demand of not only this product, but also demand for other products as well.  By conventional economic wisdom, $\mathbf{B}_t$ will have the largest entries along its diagonal because demand for a product is primarily driven by its price rather than the price of other possibly unrelated products.  Since a price increase usually leads to falling demand, it is reasonable to assume all $\mathbf{B}_t \succeq 0$ are \emph{positive-definite} (but not necessarily Hermitian), which implies that at each round:  $R_t$ is a convex function of $\mathbf{p}_t$.  The observed aggregate demands over each time period  are additionally subject to random fluctuations driven by the noise term $\bm{\epsilon}_t \in \mathbb{R}^N$.  Throughout, we suppose the noise in each round $\bm{\epsilon}_t$ is sampled i.i.d.\  from some mean-zero distribution with finite variance.      The classic analysis of \citet{Houthakker70} established that historical demand data often nicely fit a linear relationship.  A wealth of past work on dynamic pricing has also posited linear demand models, although most prior research has not considered settings where the underlying model is changing over time \cite{Keskin14, Besbes15, Cohen16, Javanmard16,Javanmard17}.

Unlike standard statistical approaches to this problem which rely on stationarity, we suppose $\mathbf{c}_t, \mathbf{B}_t$ may change every round and are possibly chosen adversarially.   This consideration is particularly important in dynamic markets where the seller faces new competitors and consumers with ever-changing preferences who are actively seeking out the cheapest prices for products \cite{Witt86}.  
 Our goal is to select prices $\mathbf{p}_1,\dots, \mathbf{p}_T$ which minimize the expected \emph{regret}   ${\mathbb{E}[R(\mathbf{p}_1,\dots, \mathbf{p}_T) - R(\mathbf{p}^*, \dots, \mathbf{p}^*)]}$ compared to always selecting the single best configuration of prices $ {\mathbf{p}^* = \argmin_{\mathbf{p} \in \mathcal{S}} \mathbb{E} \sum_{t=1}^T R_t(\mathbf{p})}$ chosen in hindsight after the revenue functions $R_t$ have all been revealed.  
 
 Low regret algorithms ensure that in the case of a stationary underlying model, our chosen prices quickly converge to the optimal choice, and in nonstationary settings, our pricing procedure will naturally adapt to the intrinsic difficulty of the dynamic revenue-optimization problem 
  \cite{Shalev11}.  
 While low (i.e.\ $o(T)$) regret is achievable using algorithms for online convex optimization with bandit feedback, the regret of existing methods is bounded below by $\Omega(\sqrt{N})$, which is undesirable large when one is dealing with a vast number of products  \cite{Dani07, Shalev11, Flaxman05}.  To attain better bounds, we adopt a low-rank structural assumption that the variation in demands changes over time only due to $d \ll N$ underlying factors.  Under this setting, we develop algorithms whose regret depends only on $d$ rather than $N$ by combining existing bandit methods with low-dimensional projections selected via online singular value decomposition.   
As far as we are aware, our main result (Theorem \ref{thm:unknownspan}) is the first online bandit optimization algorithm whose regret provably does not scale with the action-space dimensionality. 

Appendix \ref{sec:notation} provides a glossary of notation used in this paper, and all proofs of our theorems are relegated to  Appendix \ref{sec:proofs}.  
Throughout, $C$ denotes a universal constant, whose value may change from line to line (but never depends on problem-specific constants such as $T, d, r$).  

\tocless\section{Related Work}
\label{sec:relatedwork}

While bandit optimization has been successfully applied to dynamic pricing, research in this area has been primarily restricted to stationary settings \cite{Kleinberg03, Besbes09, denBoer13, Keskin14, Cohen16, Misra17}.  Most similar to our work,  \citet{Javanmard17} recently developed a bandit pricing strategy that presumes demand depends linearly on prices and product-specific features. 
 High-dimensional dynamic pricing was also addressed by \citet{Javanmard16} using sparse maximum likelihood.
 However, due to their reliance on stationarity, these approaches are less robust under evolving/adversarial environments compared with online optimization \cite{Bubeck12}. 
 
Beyond pricing, existing algorithms that combine bandits with subspace estimation \cite{Gopalan16, Djolonga13, Sen17} are solely designed for stationary (\emph{stochastic}) settings rather than general online optimization (where the reward functions can vary adversarially over time).
While the field of online bandit optimization has seen many advances since the pioneering work of Flaxman et al.\  \cite{Flaxman05}, none of the recent improvements guarantees regret that is independent of the action-space dimension \cite{Hazan14, Bubeck16}.  To our knowledge, \citet{Hazan16} is the only prior work to present online optimization algorithms whose regret depends on an intrinsic low rank structure rather than the ambient dimensionality.  
However, their approach for online learning with experts is not suited for dynamic pricing since it is restricted to settings with: full-information (rather than bandit feedback), linear and noise-free (or stationary) reward functions, and actions that are specially constrained within the probability-simplex.

\tocless\section{Low Rank Demand Model}
\label{sec:lowrank}

We now introduce a special case of model (\ref{eq:basemodel}) in which both $\mathbf{c}_t$ and $\mathbf{B}_t$  display  low-rank changes over time.
In practice, each product $i$ may be described by some vector of features $\mathbf{u}_i \in \mathbb{R}^d$ (with $d \ll N$), which determine the similarity between products as well as their baseline demands.  
A natural method to gauge similarity between products $i$ and $j$ is via their inner product $\langle \mathbf{u}_i ,  \mathbf{u}_j \rangle_\mathbf{V} = \mathbf{u}_i^T \mathbf{V} \mathbf{u}_j$ under some linear transformation of the feature-space given by $\mathbf{V} \succeq 0$.  For example, $\mathbf{u}_i$ might be a binary vector indicating that product $i$ falls into certain product-categories (where the number of categories $d$ is far less than the number of products $N$), and $\mathbf{V}$ might be a diagonal matrix specifying the cross-elasticity of demand within each product category.  In this example, $\mathbf{u}_i^T \mathbf{V} \mathbf{u}_j \cdot p_j$ would thus be the marginal effect on the demand for product $i$ that results from selecting $p_j$ as the price for product $j$.  
Many recommender systems also assume products can be described using low-dimensional latent features that govern their desirability to consumers \citep{Zhao16, Sen17}.

By introducing time-varying metric transformations $\mathbf{V}_t$, our model allows these product-similarities to evolve over time. 
Encoding the features $\mathbf{u}_i$ that represent each product as rows in a matrix $\mathbf{U} \in \mathbb{R}^{N \times d}$, 
we assume the following demand model, in which the temporal variation naturally exhibits low-rank structure:
\begin{equation}
\mathbf{q}_t = \mathbf{U} \mathbf{z}_t - \mathbf{U} \mathbf{V}_t \mathbf{U}^T \mathbf{p}_t + \bm{\epsilon}_t
\label{eq:lowrank}
\end{equation}
Here,  
the $\bm{\epsilon}_t \in \mathbb{R}^N$ again reflect statistical noise in the observed demands, the 
$\mathbf{z}_t \in \mathbb{R}^d$ explain the variation in baseline demand over time, 
and the (asymmetric) matrices $\mathbf{V}_t \in \mathbb{R}^{d \times d}$ specify latent changes in the demand-price relationship over time.  Under this model, the aggregate demand for product $i$ at time $t$ is governed by the prices of all products, weighted by their current feature-similarity to product $i$.  To ensure our revenue-optimization remains convex, we restrict the adversary to choices that satisfy $\mathbf{V}_t \succeq 0$ for all $t$.  Note that while the structural variation in our model is  assumed to be low-rank, the noise in the observed demands may be intrinsically $N$-dimensional. 
In each round, $\mathbf{p}_t$ and  $\mathbf{q}_t$ are the only quantities observed, while $\bm{\epsilon}_t, \mathbf{z}_t, \mathbf{V}_t$ all remain unknown (and we consider both cases where the product features $\mathbf{U}$ are known or unknown). 
In \S\ref{sec:empiricalrank}, we verify that our low-rank assumption accurately describes real historical demand data.

\tocless\section{Methods}
\label{sec:methods}

Our basic dynamic pricing strategy is to employ the \emph{gradient-descent without a gradient} (\flaxmanabbrev{}) online bandit optimization technique of \citet{Flaxman05}.  While a naive application of this algorithm produces regret dependent on the number of products $N$, we ensure the updates of this method are only applied in the $d$-dimensional subspace spanned by $\mathbf{U}$, which leads to regret bounds that depend only on $d$ rather than $N$.  When $\mathbf{U}$ is unknown, this subspace is simultaneously estimated online, in a somewhat similar fashion to the approach of \citet{Hazan16} for online learning with low-rank experts.  If we define $\mathbf{x} = \mathbf{U}^T \mathbf{p}   \in \mathbb{R}^d$, then 
  under the low-rank model in (\ref{eq:lowrank}) with $\mathbb{E}[\bm{\epsilon}_t] = 0$, the expected value of our revenue-objective in round $t$ can be expressed as:
\begin{align*} \mathbb{E}_{\bm{\epsilon}}[R_t(\mathbf{p})]  = \mathbf{p}^T \mathbf{U} \mathbf{V}_t \mathbf{U}^T \mathbf{p} -  \mathbf{p}^T \mathbf{U} \mathbf{z}_t  = \mathbf{x}^T \mathbf{V}_t  \mathbf{x} - \mathbf{x}^T \mathbf{z}_t  := f_t(\mathbf{x}) \numberthis
\label{eq:lowdimensionalequivalence}
\end{align*} 
As this problem's intrinsic dimensionality is only $d$, we can maximize expected revenues by merely considering a  restricted set of $d$-dimensional actions $\mathbf{x}$  and functions $f_t$ over projected constraint set:
\begin{equation}
\mathbf{U}^T(\mathcal{S}) = \big\{ \mathbf{x} \in \mathbb{R}^d : \mathbf{x} = \mathbf{U}^T \mathbf{p} \ \text{for some }  \mathbf{p}  \in \mathcal{S} \big\}
\end{equation}

\tocless\subsection{Products with Known Features}
\label{sec:known} 

In certain markets, it is clear how to featurize products \cite{Cohen16}.  Under the low-rank model in (\ref{eq:lowrank}) when $\mathbf{U}$ is given, we can apply the \algfirstabbrev{} method (Algorithm \ref{alg:first}) to select prices.  This algorithm employs subroutines \Call{FindPrice}{} and \Call{Projection}{} which both solve convex optimization problems in order to compute certain projections.  Here, $ \mathcal{B}_d =  \text{Unif}(\{\mathbf{x} \in \mathbb{R}^d : ||\mathbf{x}||_2 = 1\})$ denotes a  uniform distribution over surface of the unit sphere in $\mathbb{R}^d$.  

Intuitively, our algorithm adapts \flaxmanabbrev{} to select low-dimensional actions $\mathbf{x}_t \in \mathbb{R}^d$ at each time point, and then seeks out a feasible price vector $\mathbf{p}_t$ corresponding to the chosen $\mathbf{x}_t$.  Note that when $d \ll N$, there are potentially many price-vectors $\mathbf{p} \in \mathbb{R}^N$ that map to the same low-dimensional vector $\mathbf{x} \in \mathbb{R}^d$ via $\mathbf{U}^T$.  Out of these, we select the one that is closest to our previously-chosen prices (via \Call{FindPrice}{}), ensuring additional stability in our dynamic pricing procedure.  
In practice, the initial prices $\mathbf{p}_0$ should be selected based on external knowledge or historical demand data. 

\begin{figure}[tb]
\begin{minipage}{0.56\textwidth}

\begin{algorithm}[H]
 \caption{ \ \algfirstabbrev{} \newline { (Online Pricing Optimization with Known Features) }} 
\label{alg:first}
\begin{algorithmic}[1]
\Require  $\eta, \delta, \alpha > 0$,  $\mathbf{U} \in \mathbb{R}^{N \times d}$, initial prices $\mathbf{p}_0 \in \mathcal{S}$
\item[] \vspace*{-3mm}
\Ensure Prices $\mathbf{p}_1, \dots, \mathbf{p}_T$ to maximize revenue 
\item[] \vspace*{-2mm}
\State Set prices to $\mathbf{p}_0 \in \mathcal{S}$ and observe $\mathbf{q}_0(\mathbf{p}_0), R_0(\mathbf{p}_0)$
\item[] \vspace*{-3mm}
\State Define $\mathbf{x}_1 =  \mathbf{U}^T \mathbf{p}_0$
\item[] \vspace*{-3mm}
\NoDo
\For{ $t =1, \dots, T$: } 
\item[] \vspace*{-3mm}
\State $\bm{\xi}_t \sim \text{Unif}(\{\mathbf{x} \in \mathbb{R}^d : ||\mathbf{x}||_2 = 1\})$
\item[] \vspace*{-3mm}
\State  $\widetilde{\mathbf{x}}_t  :=  \mathbf{x}_t  + \delta \bm{\xi}_t$
\item[] \vspace*{-3mm}
\State  Set prices: {$\mathbf{p}_t = $ \textproc{FindPrice}$(\widetilde{\mathbf{x}}_t, \mathbf{U}, \mathcal{S}, \mathbf{p}_{t-1})$} \item[ \hspace*{9.6mm} and observe $\mathbf{q}_t(\mathbf{p}_t), R_t(\mathbf{p}_t)$]
\item[] \vspace*{-3mm}
\State  $\mathbf{x}_{t+1} = $ \textproc{Projection}$(\mathbf{x}_t - \eta R_t(\mathbf{p}_t) \bm{\xi}_t$, $\alpha$, $\mathbf{U}$, $\mathcal{S})$
\item[] \vspace*{-3mm}
\EndFor
\end{algorithmic}
 \end{algorithm}

\end{minipage} \hspace*{0.015\textwidth}
\begin{minipage}{0.42\textwidth}

\begin{algorithm}[H]
\caption{ \ \textsc{FindPrice}($\mathbf{x}$; $\mathbf{U}, \mathcal{S}, \mathbf{p}_{t-1}$)}\label{alg:findprice}
\begin{algorithmic}[1]
\Require   $\mathbf{x} \in \mathbb{R}^d$, $\mathbf{U} \in \mathbb{R}^{N \times d}$, 
\item[ \hspace*{10mm} convex $\mathcal{S} \subset \mathbb{R}^N$, $\mathbf{p}_{t-1} \in  \mathbb{R}^N$ ]
\item[] \vspace*{-2.2mm} 
\Ensure $\displaystyle \argmin_{\mathbf{p} \in \mathcal{S}} || \mathbf{p} - \mathbf{p}_{t-1} ||_2$ 
\item[ \hspace*{12.4mm}  subject to:  \   $\mathbf{U}^T \mathbf{p} = \mathbf{x}$]
\end{algorithmic}
 \end{algorithm}
 
  \vspace*{-6.5mm}

\begin{algorithm}[H]
\caption{ \ \textsc{Projection}($\mathbf{x}$, $\alpha$, $\mathbf{U}$, $\mathcal{S}$)}\label{alg:projecttofeasible}
\begin{algorithmic}[1]
\Require $\mathbf{x} \in \mathbb{R}^d$, $\alpha > 0$, $\mathbf{U} \in \mathbb{R}^{N \times d}$, 
\item[\hspace*{10.2mm}  convex set $\mathcal{S} \subset \mathbb{R}^N$ ]
\item[] \vspace*{-2.2mm} 
\Ensure \ $(1-\alpha) \mathbf{U}^T \widehat{\mathbf{p}}$  
\item[with \ $\displaystyle \widehat{\mathbf{p}} := \argmin_{\mathbf{p} \in \mathcal{S}} \big|\big|(1-\alpha) \mathbf{U}^T \mathbf{p} - \mathbf{x}\big|\big|_2 $]
\end{algorithmic}
\end{algorithm}

\end{minipage}
\end{figure}

Under mild conditions, Theorem \ref{cor:knownu} below states that the \algfirstabbrev{} algorithm incurs $O(T^{3/4} \sqrt{d})$ regret when product features are a priori known.  This result is derived from Lemma \ref{lem:grad} which shows that Step 7 of our algorithm corresponds (in expectation) to online projected gradient descent on a smoothed version of our objective defined as:
\begin{equation}
\widehat{f}_t(\mathbf{x}) = \mathbb{E}_{\bm{\zeta}} \big[ f_t(\mathbf{x} + \bm{\zeta} ) \big]
\label{eq:smoothed}
\end{equation}
where $\bm{\zeta}$ is sampled uniformly from within the unit sphere in $\mathbb{R}^d$, and
$f_t$ is defined in (\ref{eq:lowdimensionalequivalence}).   
We bound the regret of our pricing algorithm under the following assumptions (which ensure the revenue functions are bounded/smooth and the set of feasible prices is bounded/well-scaled): 
\vspace*{-0.5em}
\begin{itemize}  \setlength\itemsep{-0.1em}
\refstepcounter{assumptioncount} \label{assbz}
\item [(A\arabic{assumptioncount})] $||\mathbf{z}_t||_2 \le b$ \ \ for $t = 1,\dots,T$  \ 
\refstepcounter{assumptioncount} \label{assbV}
\item [(A\arabic{assumptioncount})] $||\mathbf{V}_t||_\text{op} \le b$ \ for all $t$ \ ($|| \cdot ||_\text{op} $ denotes spectral norm) 
\refstepcounter{assumptioncount} \label{assTsec}
\item [(A\arabic{assumptioncount})] $T > \frac{9}{4}d^2$
\refstepcounter{assumptioncount} \label{assUorth} 
\item [(A\arabic{assumptioncount})] $\mathbf{U}$ is an \emph{orthogonal} matrix such that $\mathbf{U}^T \mathbf{U} = \mathbf{I}_{d \times d}$
\refstepcounter{assumptioncount} \label{assballS}
\item [(A\arabic{assumptioncount})] ${\mathcal{S} = \{\mathbf{p} \in \mathbb{R}^N : ||\mathbf{p}||_2 \le r \}}$   \ \ (with $r \ge 1$)
\end{itemize}
\vspace*{-0.5em}

Requiring that the columns of $\mathbf{U}$ form an orthonormal basis for $\mathbb{R}^d$, condition (A\ref{assUorth}) can be easily enforced (when ${d < N}$) by first orthonormalizing the product features.    Note that this orthogonality condition does not restrict the overall class of models specified in (\ref{eq:lowrank}), and describes the case where the features used to encode each product are uncorrelated between products (i.e.\ a minimally-redundant encoding) and have been normalized across all products.  
To see why (A\ref{assUorth}) does not limit the allowed price-demand relationships, consider that we can re-express any (non-orthogonal) $\mathbf{U}  = \mathbf{O} \mathbf{P}$ in terms of orthogonal $\mathbf{O} \in \mathbb{R}^{N \times d}$.  The demand model in (\ref{eq:lowrank}) can then be equivalently expressed in terms of $\mathbf{z}'_t = \mathbf{P} \mathbf{z}_t, 
\mathbf{V}'_t =  \mathbf{P} \mathbf{V}_t \mathbf{P}^T$ (after appropriately redefining the constant $b$ in (A\ref{assbz})-(A\ref{assbV})), 
since: 
$  \mathbf{U} \mathbf{z}_t - \mathbf{U} \mathbf{V}_t \mathbf{U}^T \mathbf{p}_t
= \mathbf{O} \mathbf{z}'_t - \mathbf{O} \mathbf{V}'_t \mathbf{O}^T \mathbf{p}_t$.  
To further simplify our analysis, we also from now adopt (A\ref{assballS}) presuming the constraint set of feasible product-prices is a centered Euclidean ball (implying our $\mathbf{p}_t, \mathbf{q}_t$ vectors now represent appropriately shifted/scaled prices and demands).

\begin{thm} \label{cor:knownu}
Under assumptions (A\ref{assbz})-(A\ref{assballS}), if we choose $\eta = \frac{1}{b(1+d)\sqrt{T}}$, $\delta = T^{-1/4}\sqrt{\frac{d r^2 (1+r)}{9 r + 6}}$, $\alpha = \frac{\delta}{r}$, then there exists $C > 0$ such that for any $\mathbf{p} \in \mathcal{S}$:
\begin{equation*}
\mathbb{E}_{\bm{\epsilon}, \bm{\xi}} \left[ \sum_{t=1}^T R_t(\mathbf{p}_t) -  \sum_{t=1}^T R_t(\mathbf{p})  \right] \le C  br (r+1)  T^{3/4}  d^{1/2}   
\end{equation*}
for the prices $\mathbf{p}_1, \dots, \mathbf{p}_T$ selected by the \algfirstabbrev{} algorithm.
\end{thm}

Theorem \ref{thm:regretknownu} shows the same $O(T^{3/4} \sqrt{d})$ regret bound holds for the \algfirstabbrev{} algorithm under relaxed conditions solely based on the revenue functions and feasible prices rather than the specific properties of our low-rank structure assumed in (A\ref{assbz})-(A\ref{assballS}). 

\tocless\subsection{Products with Latent Features}  
\label{sec:latent}

In many settings, it is not clear how to best represent products as feature-vectors.  Once again adopting the low-rank demand model in (\ref{eq:lowrank}), we now consider the case where $\mathbf{U}$ is unknown and must be estimated.  We presume the orthogonality condition (A\ref{assUorth}) holds throughout this section (recall this does not restrict the class of allowed models), which implies $\mathbf{U}$ is both an isometry as well as the right-inverse of $\mathbf{U}^T$.  Thus,  given any low-dimensional action ${\mathbf{x} \in \mathbf{U}^T(\mathcal{S})}$, we can set the corresponding prices as ${\mathbf{p} = \mathbf{U} \mathbf{x}}$ such that  $\mathbf{U}^T \mathbf{p} =\mathbf{x}$.  Lemma \ref{lem:contract} shows that this price selection-method is feasible and corresponds to changing Step 6 in the \algfirstabbrev{} algorithm to  ${\mathbf{p}_t =  \text{\textproc{FindPrice}}(\widetilde{\mathbf{x}}_t,  \mathbf{U}, \mathcal{S}, \mathbf{0})}$, where the next price is regularized toward the origin rather than the previous price $\mathbf{p}_{t-1}$.  Because prices $\mathbf{p}_t$ are multiplied by the noise term $\bm{\epsilon}_t$ within each revenue-function $R_t$, choosing minimum-norm prices can help reduce variance in the total revenue generated by our approach.  As $\mathbf{U}$ is unknown, we instead employ an estimate $\widehat{\mathbf{U}} \in \mathbb{R}^{N \times d}$, which is always restricted to be an orthogonal matrix.  

\begin{lem} 
For any orthogonal matrix $\widehat{\mathbf{U}}$ and any ${\mathbf{x} \in  \widehat{\mathbf{U}}^T (\mathcal{S})}$, define ${\widehat{\mathbf{p}} = \widehat{\mathbf{U}} \mathbf{x} \in \mathbb{R}^N}$.
Under (A\ref{assballS}): ${\widehat{\mathbf{p}} \in \mathcal{S}}$ and $\widehat{\mathbf{p}} = $ \textproc{FindPrice}(${\mathbf{x},  \widehat{\mathbf{U}}, \mathcal{S}, \mathbf{0}}$).   
\label{lem:contract}
\end{lem}

\textbf{Product Features with Known Span.} \ 
In Theorem \ref{thm:modified},  we consider a minorly \emph{modified} \algfirstabbrev{} algorithm where price-selection in Step 6 is done using 
$\mathbf{p}_t = \widehat{\mathbf{U}} \widetilde{\mathbf{x}}_t$ rather than being regularized toward the previous price $\mathbf{p}_{t-1}$.  Even without knowing the true latent features, this result implies that the regret of our modified \algfirstabbrev{} algorithm may still be bounded independently of the number of products $N$, as long as  $\widehat{\mathbf{U}}$ accurately estimates the column span of $\mathbf{U}$.

\begin{thm}
Suppose $\text{span}(\widehat{\mathbf{U}}) = \text{span}(\mathbf{U})$, i.e.\ our orthogonal estimate has the same column-span as the underlying (rank $d$) latent  product-feature matrix.  Let  $\mathbf{p}_1, \dots, \mathbf{p}_T \in \mathcal{S}$ denote the prices selected by our modified \algfirstabbrev{} algorithm with $\widehat{\mathbf{U}}$ used in place of the underlying $\mathbf{U}$ and $\eta, \delta, \alpha$ chosen as in Theorem \ref{cor:knownu}.  Under conditions (A\ref{assbz})-(A\ref{assballS}), there exists $C > 0$ such that for any $\mathbf{p} \in \mathcal{S}$:
\begin{equation*} 
\mathbb{E}_{\bm{\epsilon}, \bm{\xi}} \left[ \sum_{t=1}^T R_t(\mathbf{p}_t) -  \sum_{t=1}^T R_t(\mathbf{p})  \right] 
\le   C  br (r+1)  T^{3/4}  d^{1/2} 
\end{equation*}
 \label{thm:modified}
\end{thm}

\textbf{Features with Unknown Span and Noise-free Demands.} \
\label{sec:nonoise}
In practice, span($\mathbf{U}$) may be entirely unknown.  If we assume the adversary is restricted to strictly positive-definite $\mathbf{V}_t \succ 0$  for all $t$ and there is no statistical noise in the observed demands (i.e.\ $\mathbf{q}_t = \mathbf{U}\mathbf{z}_t - \mathbf{U} \mathbf{V}_t \mathbf{U}^T \mathbf{p}_t$ in each round), then Lemma \ref{lem:zerochance} below shows we can ensure  span($\mathbf{U}$) is revealed within the first $d$ observed demand vectors by simply adding a minuscule random perturbation to all of our initial prices selected in the first $d$ rounds.
Thus, even without knowing the latent product feature subspace, an absence of noise in the observed demands enables us to realize a low regret pricing strategy via the same modified \algfirstabbrev{} algorithm (applied after the first $d$ rounds).

\begin{lem}
Suppose that for $t = 1, \dots, T$: $\bm{\epsilon}_t = 0$ and $\mathbf{V}_t \succ 0$.  If each $\mathbf{p}_t$ is independently uniformly distributed within some (uncentered) Euclidean ball of strictly positive radius, then span$(\mathbf{q}_1,\dots, \mathbf{q}_d) = \text{span}(\mathbf{U})$ almost surely.  
\label{lem:zerochance}
\end{lem}

\textbf{Features with Unknown Span and Noisy Demands.} \  
When the observed demands are noisy and span$(\mathbf{U})$ is unknown, we select prices using the \algsecondabbrev{} algorithm on the next page.  
The approach is similar to our previous \algfirstabbrev{} algorithm, except we now additionally maintain a changing estimate of the latent product features' span.  Our estimate is updated in an online fashion via an averaged singular value decomposition (SVD) of the previously observed demands.

\begin{algorithm}[tb]
 \caption{\algsecondabbrev{} (Online Pricing Optimization with Latent Features) }
 \label{alg:second}
\begin{algorithmic}[1]
 \Require $\eta, \delta, \alpha > 0$, rank $d \in [1, N]$, initial prices $\mathbf{p}_0 \in \mathcal{S}$ 
\Ensure  Prices $\mathbf{p}_1, \dots, \mathbf{p}_T$ to maximize overall revenue
\vspace*{2mm}
\State  Initialize $\widehat{\mathbf{Q}}$ as $N \times d$ matrix of zeros
\State Initialize $\widehat{\mathbf{U}}$ as random $N \times d$ orthogonal matrix 
\State Set prices to $\mathbf{p}_0 \in \mathcal{S}$ and observe $\mathbf{q}_0(\mathbf{p}_0), R_0(\mathbf{p}_0)$ 
\State Define $\mathbf{x}_1 = \widehat{\mathbf{U}}^T \mathbf{p}_0$
\NoDo
\For{$t =1, \dots, T$:} 
\State  $\widetilde{\mathbf{x}}_t  :=  \mathbf{x}_t  + \delta \bm{\xi}_t$, \ \ $\bm{\xi}_t \sim \text{Unif}(\{\mathbf{x} \in \mathbb{R}^d : ||\mathbf{x}||_2 = 1\})$
\State Set prices: $\mathbf{p}_t = \widehat{\mathbf{U}} \widetilde{\mathbf{x}}_t$  and observe  $\mathbf{q}_t(\mathbf{p}_t), R_t(\mathbf{p}_t)$
\State $\mathbf{x}_{t+1} = $ \textproc{Projection}$(\mathbf{x}_t - \eta R_t(\mathbf{p}_t) \bm{\xi}_t$, $\alpha$, $\widehat{\mathbf{U}}$, $\mathcal{S})$
\State With $j = 1 + [(t-1) \text{ mod } d]$, $k= \text{floor}(t/d)$, update: $\widehat{\mathbf{Q}}_{*,j} \leftarrow \frac{1}{k} \mathbf{q}_{t} +  \frac{k-1}{k}\widehat{\mathbf{Q}}_{*,j}$ 
\State Set columns of $\widehat{\mathbf{U}}$ as top $d$ left singular vectors of $\widehat{\mathbf{Q}}$  
\EndFor
\end{algorithmic}
\end{algorithm}

Step 9 in our \mbox{\algsecondabbrev{}} algorithm corresponds to online averaging of the currently observed demand vector $\mathbf{q}_t$ with the historical observations stored in the $j^\text{th}$ column of matrix $\widehat{\mathbf{Q}}$.
After computing the singular value decomposition of $\widehat{\mathbf{Q}} = \widetilde{\mathbf{U}} \widetilde{\mathbf{S}} \widetilde{\mathbf{V}}^T $, Step 10  is performed by setting $\widehat{\mathbf{U}}$ equal to the first $d$ columns of $\widetilde{\mathbf{U}}$ (presumed to be the indices corresponding to the largest singular values in $\widetilde{\mathbf{S}}$).  Since $\widehat{\mathbf{Q}}$ is only minorly changed within each round, the update operation in Step 10 can be computed more efficiently by leveraging existing fast SVD-update procedures \cite{Brand06, Stange08}.  Note that by their definition as singular vectors, the columns of $\widehat{\mathbf{U}}$ remain orthonormal throughout the execution of our algorithm.

To quantify the regret incurred by this algorithm, we assume the noise vectors $\bm{\epsilon}_t$ follow a sub-Gaussian distribution for each $t = 1,\dots, T$.  The assumption of sub-Gaussian noise is quite general, covering common settings where the noise is Gaussian, bounded, of strictly log-concave density, or any finite mixture of sub-Gaussian variables \cite{Mueller18}.  
Intuitively, the averaging in step 9 of our \mbox{\algsecondabbrev{}} algorithm ensures statistical concentration of the noise in our observed demands, such that the true column span of the underlying $\mathbf{U}$ may be better revealed.  More concretely, if we let $\mathbf{s}_t = \mathbf{z}_t - \mathbf{V}_t \mathbf{U}^T \mathbf{p}_t$ and $\mathbf{q}^*_t = \mathbf{U} \mathbf{s}_t$, then the observed demands can be written as: $\mathbf{q}_t = \mathbf{q}^*_t   + \bm{\epsilon}_t$, where $\mathbf{q}^*_t$ are the (unobserved) expected demands at our chosen prices.  
Thus, the $j^{\text{th}}$ column  of  $\widehat{\mathbf{Q}}$ at round $T$ is given by:   
\begin{align*}
& \hspace*{-2mm} \widehat{\mathbf{Q}}_{*,j}  = \widebar{\mathbf{Q}}^*_{*,j}+ \frac{1}{| \mathcal{I}_j |} \hspace*{-0.7mm} \sum_{i \in \mathcal{I}_j}  \hspace*{-0.1mm} \bm{\epsilon}_{i}, 
\text{ with }
 \widebar{\mathbf{Q}}^*_{*,j}   = \frac{1}{| \mathcal{I}_j |} \mathbf{U} \hspace*{-0.8mm}  \sum_{i\in  \mathcal{I}_j} \hspace*{-0.1mm}  \mathbf{s}_{i}
 \numberthis \label{eq:qbar} 
 \end{align*}  
 where we assume for notational simplicity that $T$ is divisible by $d$ and define $\mathcal{I}_j = \{ j + d(i-1) : i =1,\dots, \frac{T}{d} \}$ (so $|\mathcal{I}_j| = \frac{T}{d}$).  
 Because the average $ \frac{1}{|\mathcal{I}_j|} \sum_{i\in \mathcal{I}_j}  \bm{\epsilon}_{i}$ exhibits  concentration of measure, results from random matrix theory imply that the span-estimator obtained from the first $d$ singular vectors of  $\widehat{\mathbf{Q}}$ in Step 10 of our \mbox{\algsecondabbrev{}} algorithm will rapidly converge to the column span of $\widebar{\mathbf{Q}}^* \in \mathbb{R}^{N \times d}$, a matrix of averaged underlying expected demands.  This is useful since $\widebar{\mathbf{Q}}^*$ shares the same span as the underlying $\mathbf{U}$.

Theorem \ref{thm:unknownspan} below shows that our \algsecondabbrev{} algorithm achieves low-regret in the setting of unknown product features with noisy demands, and the regret again depends only on the intrinsic rank $d$ (rather than the number of products $N$).

\begin{thm} \label{thm:unknownspan} 
For unknown $\mathbf{U}$, let $\mathbf{p}_1, \dots, \mathbf{p}_T$ be the prices selected by the \algsecondabbrev{} algorithm with $\eta, \delta, \alpha$ set as in Theorem \ref{cor:knownu}.  Suppose $\bm{\epsilon}_t$ follows a sub-Gaussian$(\sigma^2 )$ distribution and has statistically i.i.d.\ dimensions for each $t$.   If  (A\ref{assbz})-(A\ref{assballS}) hold, then there exists ${C > 0}$ such that for any $\mathbf{p} \in \mathcal{S}$:
\begin{equation*}
\mathbb{E}_{\bm{\epsilon}, \bm{\xi}} \left[ \sum_{t=1}^T R_t(\mathbf{p}_t) -  \sum_{t=1}^T R_t(\mathbf{p})  \right] 
\le C Q rb (4r + 1) d T^{3/4} 
\end{equation*}
Here, $Q = \max\left\{1, \sigma^2 \left(\frac{2\sigma_1 + 1}{\sigma_d^2}\right) \right\}$ with $\sigma_1$ (and $\sigma_d$) defined as the largest (and smallest) nonzero singular values of the underlying rank $d$ matrix $\widebar{\mathbf{Q}}^*$ defined in (\ref{eq:qbar}).
\end{thm}

Our proof of this result relies on standard random matrix concentration inequalities \cite{Vershynin12} and Theorem \ref{thm:yu}, a useful variant of the Davis-Kahan theory introduced by \citet{Yu15}.  Intuitively, we show that span($\mathbf{U}$) can be estimated to sufficient accuracy within sufficiently few rounds, and then follow similar reasoning to the proof of Theorem \ref{thm:modified}. 
Note that the regret in Theorem \ref{thm:unknownspan} depends on the constant $Q$ whose value is determined by the noise-level $\sigma$ and the extreme singular values of $\widebar{\mathbf{Q}}^*$ defined in (\ref{eq:qbar}).  In general, these quantities thus measure just how adversarial of an environment the seller is faced with. 
For example, when the underlying low-rank variation is of much smaller magnitude than the noise in our observations, it will be difficult to accurately estimate the span of the latent product features.  
In control theory, a signal-to-noise expression similar to $Q$ has also been recently proposed to quantify the intrinsic difficulty of system identification for the linear quadratic regulator \cite{Dean17}.  A basic setting in which $Q$ can be explicitly bounded is illustrated in Appendix \ref{sec:imprecise}, where we suppose the underlying demand model parameters can only be imprecisely controlled by an adversary over time.

\tocless\section{Experiments}
\label{sec:results}

We evaluate the performance of our methodology in settings where noisy demands are generated according to equation (\ref{eq:lowrank}), and the underlying structural parameters of the demand curves are randomly sampled from Gaussian distributions (details in Appendix \ref{sec:expdetails}).  Throughout, $\mathbf{p}_t$ and $\mathbf{q}_t$ represent rescaled rather than absolute prices/demands, such that the feasible set $\mathcal{S}$ can be simply fixed as a centered sphere of radius $r = 20$.  Noise in the (rescaled) demands for each individual product is always sampled as: $\mathbf{\epsilon}_t \sim N(0, 10 )$. 
 
Our proposed algorithms are compared against the GDG online bandit algorithm of \citet{Flaxman05}, as well as a simple explore-then-exploit ($\mathrm{Explo}^\mathrm{re}_\mathrm{it}$) technique.  The latter method randomly samples $\mathbf{p}_t$ during the first $T^{3/4}$ rounds (uniformly over $\mathcal{S}$) and for all remaining rounds, $\mathbf{p}_t$  is fixed at the best price vector found during exploration.  $\mathrm{Explo}^\mathrm{re}_\mathrm{it}$ reflects a standard pricing technique: initially experiment with prices and eventually settle on those that previously produced the most  profit.

\begin{figure*}[tb] \centering
\vspace*{-0.1mm} 
\begin{minipage}{0.49\textwidth}
\begin{overpic}[width=\textwidth]{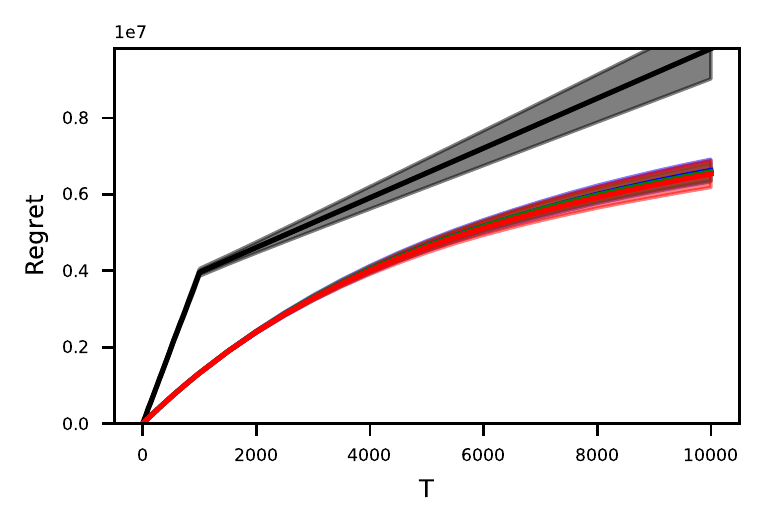}
 \put (25,64){{\footnotesize \textbf{(A)} \ $N= 10, d = 10$, Model = \S\ref{sec:stationary}}}
 \end{overpic}
 \end{minipage} \hspace*{0.005\textwidth}
\begin{minipage}{0.49\textwidth}
\begin{overpic}[width=\textwidth]{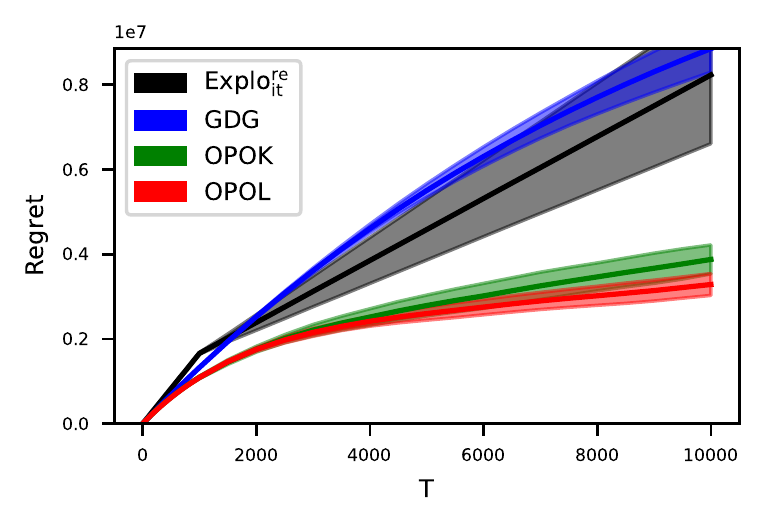} 
 \put (25, 64){\footnotesize \textbf{(B)} \ $N= 100, d = 10$, Model = \S\ref{sec:stationary}} 
\end{overpic} 
 \end{minipage}
\\[-0.4em] 

\begin{minipage}{0.49\textwidth}
\begin{overpic}[width=\textwidth]{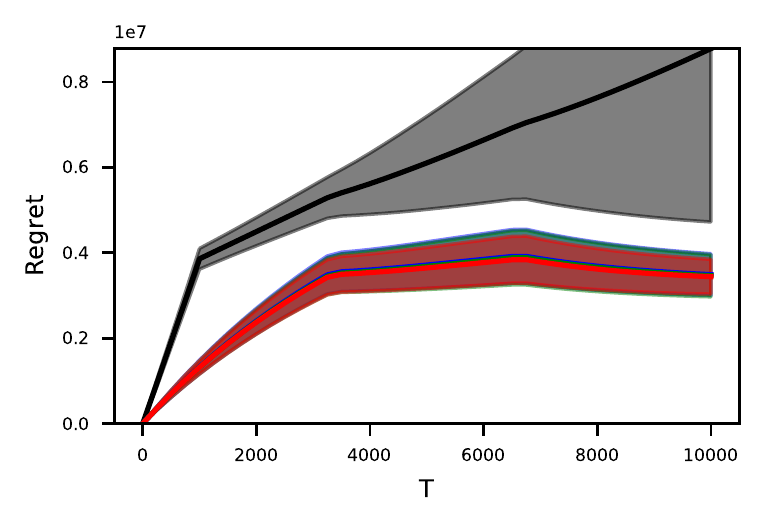}
 \put (25,64){{\footnotesize \textbf{(C)} \ $N= 10, d = 10$, Model = \S\ref{sec:shock}}}
 \end{overpic}
  \end{minipage} \hspace*{0.005\textwidth}
  \begin{minipage}{0.49\textwidth}
\begin{overpic}[width=\textwidth]{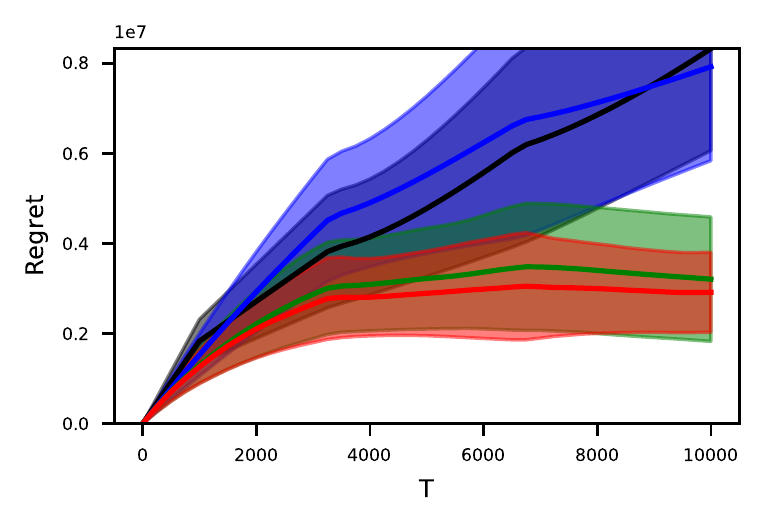} 
 \put (25, 64){\footnotesize \textbf{(D)} \ $N= 100, d = 10$, Model = \S\ref{sec:shock}} 
\end{overpic} 
 \end{minipage}
\\[-0.4em] 
 \hspace*{-6mm}
 \begin{minipage}{0.47\textwidth}
\begin{overpic}[width=1.1\textwidth]{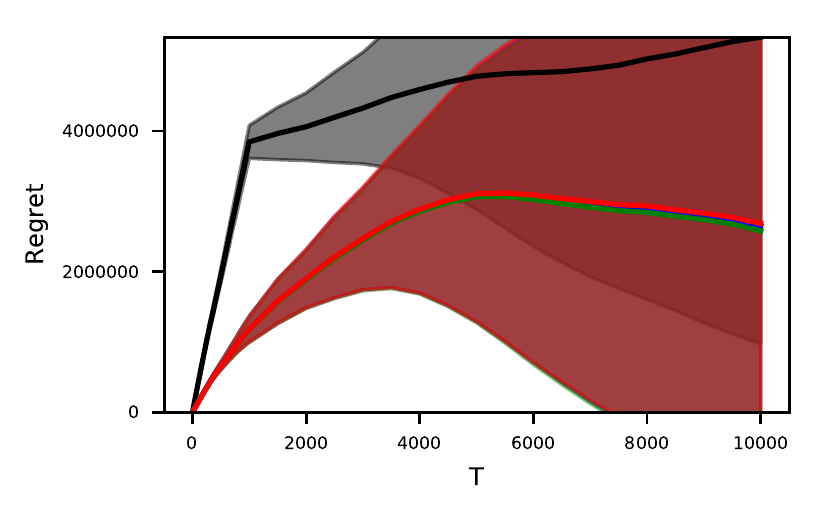}
 \put (26,60){{\footnotesize \textbf{(E)} \ $N= 10, d = 10$, Model = \S\ref{sec:drifting}}}
 \end{overpic}
  \end{minipage}  \hspace*{0.035\textwidth}
  \begin{minipage}{0.47\textwidth}
 \hspace*{-3.5mm}
\begin{overpic}[width=1.1\textwidth]{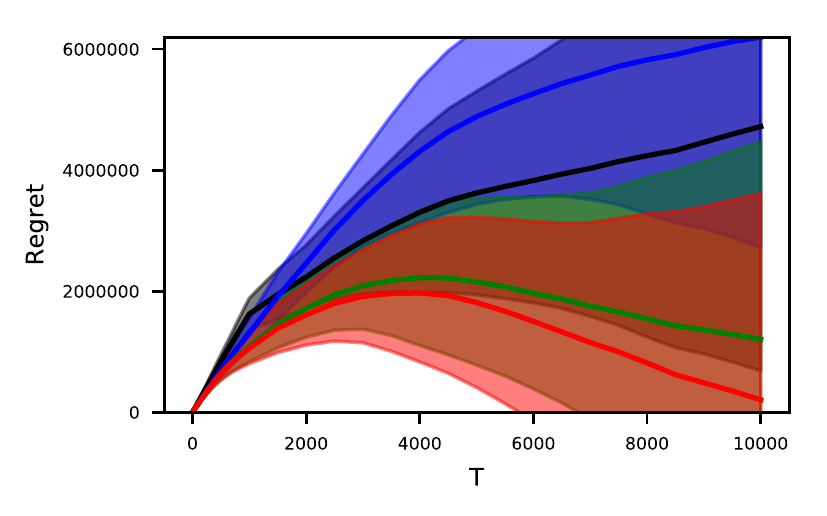} 
 \put (27.2, 61.4){\footnotesize \textbf{(F)} \ $N= 100, d = 10$, Model = \S\ref{sec:drifting}} 
\end{overpic} 
 \end{minipage}
\vspace*{-3mm}
\caption{Average cumulative regret (over 10 repetitions with standard-deviations shaded) of various pricing strategies when underlying demand model is: \textbf{(A)-(B)} stationary over time, \textbf{(C)-(D)}: altered by structural shocks at times $T/3$ and $2T/3$, \textbf{(E)-(F)}: drifting over time.}
\label{fig:regrets}
\vspace*{-1.8mm}
\end{figure*}

\tocless\subsection{Stationary Demand Model}
\label{sec:stationary}
First, we consider a stationary setting where underlying structural parameters ${\mathbf{z}_t, = \mathbf{z}}$, ${\mathbf{V}_t = \mathbf{V}}$ remain fixed.  Before each experiment, we sample the entries of ${\mathbf{z}, \mathbf{V}}$ independently as ${\mathbf{z}_{ij} \sim N(100,20)}$, ${\mathbf{V}_{ij} \sim N(0,2)}$,  and $\mathbf{U}$ is fixed as a random sparse binary matrix that reflects which of $d$ possible categories each product belongs to.  Subsequently, we orthogonalize the columns of $\mathbf{U}$ and project $\mathbf{V}$ into $\mathcal{V} = \{\mathbf{V} :  \mathbf{V}^T + \mathbf{V} \succeq \lambda \mathbb{I} \}$ with $\lambda = 10$ to ensure positive definite  cross-product price elasticities.  
Here, $ \mathbf{z}, \mathbf{V}, \lambda$ are chosen to reflect properties of real-world demand curves: different products' baseline demands and elasticities should be highly diverse (wide range of $z$), and prices should significantly influence demands such that price-increases severely decrease demand and affect
demand for the same product more than other products (large value of $\lambda$, which in turn induces large values for certain entries of $\mathbf{V}$). 
We find the optimal price vector does not lie near the boundary of $\mathcal{S}$ ($||\mathbf{p}^*||_2 \approx 8$ rather than 20), which shows that prices strongly influence demands under our setup.

Figures \ref{fig:regrets}A and \ref{fig:regrets}B show that our \algfirstabbrev{} and \algsecondabbrev{}  algorithms are greatly superior to GDG when the dimensionality $N$ exceeds the intrinsic rank $d$.  When $N = d$ (no low-rank structure to exploit), our \algfirstabbrev{}/\algsecondabbrev{} algorithms closely match GDG (blue, green, and red curves overlap).  Note that in this case: GDG and OPOK are nearly mathematically equivalent (same regret bound applies to both, but their empirical performance slightly differs in this case due to the internal stochasticity of each bandit algorithm), as are OPOL and OPOK (since $d= N$ implies $\widehat{\mathbf{U}}$ is an orthogonal $N \times N$ matrix and hence invertible).  
For small $N$, all online bandit optimization techniques outperform $\mathrm{Explo}^\mathrm{re}_\mathrm{it}$, but GDG  scales poorly to large $N$ unlike our methods.  
 Interestingly, \algsecondabbrev{} (which must infer latent product features alongside the pricing strategy) performs slightly better than the \algfirstabbrev{} approach, which has access to the ground-truth  features.  This is because in the presence of noise, our SVD-computed features can more robustly represent the subspace where projected pricing variation can maximally impact the overall observed demands.  In contrast, the dimensionality-reduction in \algfirstabbrev{} does not lead to any denoising.

\tocless\subsection{Model with Demand Shocks}
\label{sec:shock}
Next, we study a non-stationary setting where the underlying demand model changes drastically at times $T/3$ and $2T/3$.  At the start of each period $[0, T/3]$, $[T/3, 2T/3]$, $[2T/3, T]$: we simply redraw the underlying structural parameters $\mathbf{z}_t, \mathbf{V}_t$ from the same Gaussian distributions used for the stationary setting.  Figures \ref{fig:regrets}C and \ref{fig:regrets}D show that our bandit techniques quickly adapt to the changes in the underlying demand curves. 
The regret of the bandit algorithms decreases over time, indicating they begin to outperform the optimal fixed price chosen in hindsight  (recall that our bandits may vary price over time, whereas regret is measured against the best fixed price-configuration which may fare much worse than a dynamic schedule in nonstationary environments).  
Once again, our low-rank methods achieve low regret for a large number of products unlike the existing approaches, while retaining the same strong performance as the GDG algorithm in the absence of low-rank structure.  

\tocless\subsection{Drifting Demand Model}
\label{sec:drifting}
Finally, we consider another non-stationary setting where underlying demand curves slowly change over time.  Here, the underlying structural parameters  $\mathbf{z}_t, \mathbf{V}_t$ are initially drawn from the same previously used Gaussian distributions at $t=0$, but then begin to stochastically drift over time according to:
$\mathbf{z}_{t+1} = \mathbf{z}_t + \mathbf{w}, \ \ \mathbf{V}_{t+1} = \Pi_\mathcal{V}(\mathbf{V}_t + \mathbf{W})
$.  Here, the entries of $\mathbf{w}$ and $\mathbf{W}$ are i.i.d.\ samples from $N(0,1)$ and $N(0, 0.1)$ distributions, respectively, and $\Pi_\mathcal{V} $ denotes the projection of a matrix into the strongly positive-definite set $\mathcal{V}$ we previously defined.  Figures \ref{fig:regrets}E and \ref{fig:regrets}F illustrate how our bandit pricing approach can adapt to ever-changing demand curves.  Again, our low-rank methods exhibit much stronger performance than GDG and $\mathrm{Explo}^\mathrm{re}_\mathrm{it}$  in the settings with many products.

\tocless\subsection{Misspecified Demand Model}
\label{sec:wrongmodel}

Appendix \ref{sec:misresults} investigates the robustness of our algorithms in misspecified settings with full-rank or log-linear demands, where the assumptions of our demand model are explicitly violated.    Even in the absence of explicit low-rank structure, running the \algsecondabbrev{}  algorithm with low values of $d$ substantially outperforms other pricing strategies (Figure \ref{fig:highrankregrets}).  These empirical results suggest that our \algsecondabbrev{} algorithm is practically useful for various high-dimensional pricing problems, beyond those that exactly satisfy the low-rank/linearity assumptions in (\ref{eq:lowrank}).  

 \begin{figure*}[tb] \centering
\begin{overpic}[width=0.49\textwidth]{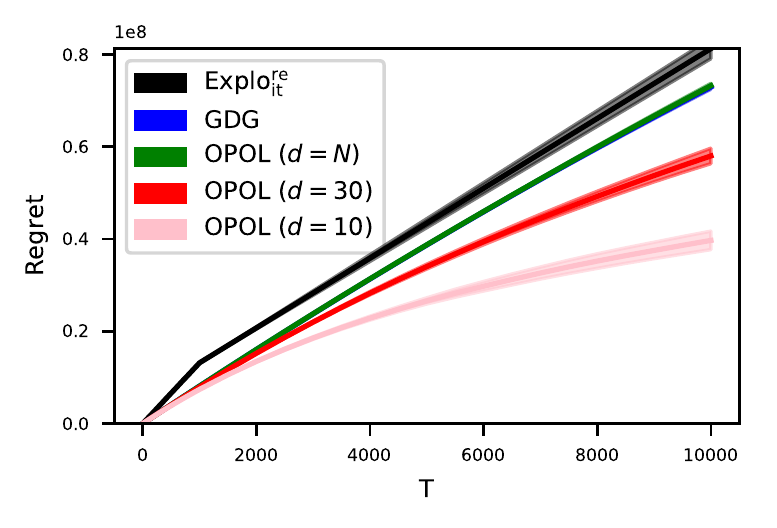}
 \put (15,68){{\footnotesize \textbf{(A)} \ Model (\ref{eq:basemodel})} without temporal change}
 \end{overpic} \hspace*{0mm}
\begin{overpic}[width=0.49\textwidth]{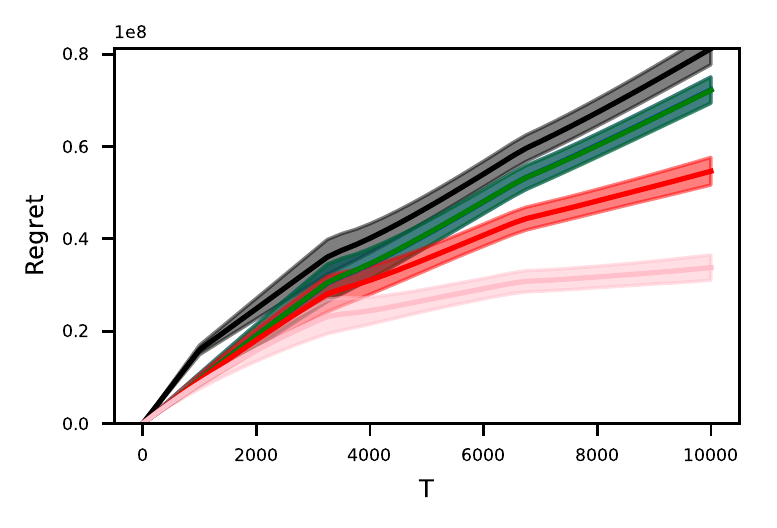} 
 \put (20, 68){\footnotesize \textbf{(B)} \ Model (\ref{eq:basemodel}) with demand shocks} 
\end{overpic}  \hspace*{1mm}
\vspace*{-4mm}
\caption{Regret of pricing strategies (for $N= 100$) when underlying demand model has no low-rank structure (see Appendix \ref{sec:misresults}) and is:  \textbf{(A)}  stationary, \textbf{(B)} altered by shocks at $T/3$ and $2T/3$ as in \S\ref{sec:shock}.  } 
\label{fig:highrankregrets}
\end{figure*}

\tocless\subsection{Rank of Historical Demand Data}
\label{sec:empiricalrank}

While the aforementioned robustness analysis indicates our approach works well even when key assumptions are violated, 
it remains of interest whether our assumptions accurately describe actual demand variation for real products.  
One key implication of our assumptions in (\ref{eq:lowrank}) is that the $N \times T$ matrix $\mathbf{Q} = [\mathbf{q}_1; \mathbf{q}_2;  \dots;\mathbf{q}_T]$, whose columns contain the observed demands in each round, should be approximately low-rank when there is limited noise in the demand-price relationship.
This is because under our assumptions, $\mathbf{q}_1,\dots, \mathbf{q}_T$ only span a $d$-dimensional subspace in the absence of noise (see proof of  Lemma \ref{lem:zerochance}).

Here, we study historical demand data\footnote{Historical demand data obtained from: \url{www.kaggle.com/c/grupo-bimbo-inventory-demand/}} for 1,340 products sold at various prices over 7 weeks by the baking company Grupo Bimbo.   
Using this data, we form a matrix $\mathbf{Q}$ whose columns contain the total weekly demands for each product across all stores.
The SVD of $\mathbf{Q}$ reveals the following percentages of variation in the observed demands are captured within the top $k$ singular vectors: $k = 1$: $ 97.1\%$, $k = 2$: $99.1\%$, $k = 3$: $99.9\%$.  
This empirical analysis thus suggests that our low-rank assumption on the expected demand variation remains reasonable in practice.

\tocless\section{Discussion}
\label{sec:discussion}

By exploiting a low-rank structural condition that naturally emerges in dynamic pricing problems, this work introduces an online bandit optimization algorithm whose regret provably depends only on the intrinsic rank of the problem rather than the ambient dimensionality of the action space.  Our low-rank bandit approach to dynamic pricing scales to a large number of products with intercorrelated demand curves, even if the underlying demand model varies over time in an adversarial fashion.  When applied to various high-dimensional dynamic pricing systems involving stationary, fluctuating, and misspecified demand curves, our approach empirically outperforms standard bandit methods.  
Future extensions of this work could include adaptations for predictable sequences in which future demands can be partially forecasted \cite{Rakhlin13}, or generalizing our convex formulation and linear demand model to more general subspace structures \cite{Hazan16b}.  

\clearpage

\nocite{Rigollet15}
\nocite{Rudelson08}
\bibliography{banditPricing}
\bibliographystyle{plainnat}

\clearpage \newpage
\beginsupplement

\appendix

\setcounter{lem}{0}
\setcounter{thm}{0}
\setcounter{cor}{0}
\renewcommand{\thelem}{\Alph{section}.\arabic{lem}}
\renewcommand{\thethm}{\Alph{section}.\arabic{thm}}
\renewcommand{\thecor}{\Alph{section}.\arabic{cor}}

\setcounter{page}{1}
\pagenumbering{arabic}
\setlength{\footskip}{20pt}  

\renewcommand\thefigure{\thesection.\arabic{figure}}    
\setcounter{figure}{0}

\renewcommand\thetable{\thesection.\arabic{table}}    
\setcounter{table}{0}

\newcommand{\toptitlebar}{ 
  \hrule height 4pt
  \vspace{0.25in}
  \vspace{-\parskip}
}

\newcommand{\bottomtitlebar}{
  \vspace{0.29in}
  \vspace{-\parskip}
  \hrule height 1pt
  \vspace{0.09in}
}

\vbox{%
    \vspace{0.1in}
    \toptitlebar
    \centering
    {\LARGE\bf Supplementary Material: \ Low-Rank Bandit Methods for High-Dimensional Dynamic Pricing
     \par}
    \bottomtitlebar
    \vspace{0.3in}
  }

 \tableofcontents

\clearpage

\section{Proofs and Auxiliary Theoretical Results} 
\label{sec:proofs}

 \begin{lem}  For $\mathbf{p} \in \mathbb{R}^N$ with $\mathbf{U}^T \mathbf{p} = \mathbf{x} + \delta \bm{\xi} \in \mathbb{R}^d$, $\bm{\xi}  \sim \text{Unif}(\{x \in \mathbb{R}^d : ||\mathbf{x}||_2 = 1\}):$
 $$
  \frac{ \partial \widehat{f}_t }{ \partial \mathbf{x} } = \frac{d}{\delta} \cdot \mathbb{E}_{\bm{\epsilon},\bm{\xi}} \big[ {R_t(\mathbf{p}) \bm{\xi}} \big]
$$
 \label{lem:grad} 
  \addcontentsline{toc}{subsubsection}{Lemma \ref{lem:grad}}
 \end{lem}
 
 \begin{proof}
 Since we have: $\mathbb{E}_{\bm{\epsilon}} [R_t(\mathbf{p})] = f_t(\mathbf{x} + \delta  \bm{\xi})$,
 this result follows directly from Lemma 2.1 in \citet{Flaxman05}.  
\end{proof}

\begin{thm}[\citeauthor{Flaxman05}, \citeyear{Flaxman05}]  Suppose for $t = 1,\dots, T$, each $f_t \in [-B, B]$ is a convex, $L$-Lipschitz function of $\mathbf{x} \in \mathbb{R}^d$, and the set of feasible actions $\mathcal{U} \subset \mathbb{R}^d$ is convex, with Euclidean balls of radius $r_{\smallup}$ and $r_{\smalldown}$ containing and contained-within $\mathcal{U}$, respectively.  Let $\mathbf{x}_1, \dots, \mathbf{x}_T \in \mathbb{R}^d$  denote the iterates of the GDG algorithm applied to $f_1,\dots, f_T$ (i.e.\ online projected stochastic gradient descent applied to the $\widehat{f}_t$ as defined in (\ref{eq:smoothed})).
If we choose $\eta, \delta, \alpha$ as in Theorem \ref{thm:regretknownu}, then:
\begin{equation*}
\mathbb{E} \left[ \sum_{t=1}^T f_t(\mathbf{x}_t)  - \min_{\mathbf{x} \in \mathcal{U}} \sum_{t=1}^T f_t(\mathbf{x})  \right]  
\le 2 T^{3/4} \sqrt{3 B r_{\smallup} \bigg(L + \frac{B}{r_{\smalldown}}\bigg) d} 
\end{equation*}   
\label{thm:flax}  
 \addcontentsline{toc}{subsubsection}{Theorem \ref{thm:flax}}
\end{thm}  

\subsection{Alternative OPOK Regret Bound}

We provide another bound on the regret of our pricing algorithm that is similar to Theorem \ref{cor:knownu}, but only relies on direct properties of the prices and revenue functions rather than properties of our assumed low-rank structure.

The following assumptions are adopted (revenue functions are bounded/smooth, and the set of feasible prices is bounded/well-scaled): 
\begin{itemize}  \setlength\itemsep{-0.1em}
\refstepcounter{assumptioncount} \label{assBalls}
 \item [(A\arabic{assumptioncount})] $\mathbf{U}^T(\mathcal{S})$ contains a Euclidean ball of radius $r_{\smalldown}$ and is contained within a ball of radius ${r_{\smallup} \ge r_{\smalldown}}$
\refstepcounter{assumptioncount} \label{assTfirst}
\item [(A\arabic{assumptioncount})] $T > \big(\frac{3dr_{\smallup}}{2r_{\smalldown}}\big)^2$  \hspace*{1.5mm} (the number of  pricing rounds is  large)
\refstepcounter{assumptioncount} \label{assbRobj}
\item [(A\arabic{assumptioncount})] $ \big| \mathbb{E} [R_t(\mathbf{p})] \big| \le B$ for all $\mathbf{p} \in \mathcal{S}$, $t = 1,\dots,T$ 
\refstepcounter{assumptioncount} \label{assLip}
\item [(A\arabic{assumptioncount})] $f_t (\mathbf{x}) $ is $L$-Lipschitz over $\mathbf{x} \in \mathbf{U}^T(\mathcal{S})$   for $t = 1,\dots,T$
\end{itemize}

\begin{thm}  \label{thm:regretknownu}
If conditions (A\ref{assBalls})-(A\ref{assLip}) are met and we choose $\eta = \frac{r_{\smallup}}{B\sqrt{T}}$, ${\delta = T^{-1/4}\sqrt{\frac{B d r_{\smallup} r_{\smalldown}}{3(Lr_{\smalldown} + B)}}}$, $\alpha = \frac{\delta}{r_{\smalldown}}$, then there exists $C > 0$ such that for any $\mathbf{p} \in \mathcal{S}$:
\begin{equation*} 
\mathbb{E}_{\bm{\epsilon}, \bm{\xi}} \hspace*{-2pt} \left[ \sum_{t=1}^T R_t(\mathbf{p}_t) - \hspace*{-1pt}  \sum_{t=1}^T R_t(\mathbf{p}) \hspace*{-1pt}  \right] \hspace*{-2pt} \le C T^{3/4} \hspace*{-1pt} \sqrt{B d r_{\smallup} \hspace*{-1pt} \bigg( \hspace*{-1pt} L + \frac{B}{r_{\smalldown}} \bigg) \hspace*{-2pt}}
\end{equation*}
for the prices $\mathbf{p}_1, \dots, \mathbf{p}_T$ selected by the \algfirstabbrev{} algorithm.
 \addcontentsline{toc}{subsubsection}{Theorem \ref{thm:regretknownu}}
\end{thm}

\begin{proof} 
Condition (A\ref{assbRobj}) implies the range of $f_t$ bounded by $B$ over $\mathbf{x} \in \mathbf{U}^T(\mathcal{S})$.  Recall that each $f_t$ is a convex function of $\mathbf{x}$ (as we required  each $\mathbf{V}_t \succeq 0$) and for any $\mathbf{p} \in \mathcal{S}$, we can define $\mathbf{x} = \mathbf{U}^T \mathbf{p} \in \mathbf{U}^T (\mathcal{S})$ such that: $\mathbb{E}_{\bm{\epsilon}}[R_t(\mathbf{p})] = f_t(\mathbf{x})$.  Since convexity of $\mathcal{S}$ implies $\mathbf{U}^T (\mathcal{S})$ is also convex, the proof of our result immediately follows from Theorem 3.3 in \citet{Flaxman05}, which is also restated here as Theorem  \ref{thm:flax}.  Finally, we note that since both $\mathcal{S}$ and $\mathbf{U}^T(\mathcal{S})$ are  convex, our choice of $\eta, \delta, \alpha$ ensures $\widetilde{\mathbf{x}}_t \in \mathbf{U}^T(\mathcal{S})$ and hence $\mathbf{p}_t \in \mathcal{S}$ for all $t$.
\end{proof}

\subsection{Proof of Theorem \ref{cor:knownu}}

\begin{customthm}{\ref{cor:knownu}}
Under assumptions (A\ref{assbz})-(A\ref{assballS}), if we choose $\eta = \frac{1}{b(1+d)\sqrt{T}}$, $\delta = T^{-1/4}\sqrt{\frac{d r^2 (1+r)}{9 r + 6}}$, $\alpha = \frac{\delta}{r}$, then there exists $C > 0$ such that for any $\mathbf{p} \in \mathcal{S}$:
\begin{equation*}
\mathbb{E}_{\bm{\epsilon}, \bm{\xi}} \left[ \sum_{t=1}^T R_t(\mathbf{p}_t) -  \sum_{t=1}^T R_t(\mathbf{p})  \right] \le C  br (r+1)  T^{3/4}  d^{1/2}   
\end{equation*}
for the prices $\mathbf{p}_1, \dots, \mathbf{p}_T$ selected by the \algfirstabbrev{} algorithm.
\end{customthm}

\begin{proof}
We show that  (A\ref{assbz})-(A\ref{assballS}) imply Theorem \ref{thm:regretknownu} holds with  ${r_{\smallup} =  r_{\smalldown} = r}$,  $B = rb(1 + r)$, and $L = (2r+1)b$.  
Bounding and simplifying the inequality then produces the desired result.
Note that (A\ref{assbRobj}) holds since:
\begin{align*}
f_t(\mathbf{x}) & 
\le ||\mathbf{x}||^2_2 || \mathbf{V}_t ||_\text{op}  + || \mathbf{z}_t ||_2 || \mathbf{x} ||_2 \le r^2 b + r b
\end{align*}
We also have Lipschitz continuity as required in (A\ref{assLip}), since for all $\mathbf{x} \in \mathbf{U}^T (\mathcal{S})$:
\begin{align*} 
 || \nabla_\mathbf{x} f_t(\mathbf{x}) ||_2  = ||( \mathbf{V}_t^T  +\mathbf{V}_t) \mathbf{x} - \mathbf{z}_t ||_2 
\le \ 2|| \mathbf{V}_t||_\text{op} || \mathbf{x} ||_2 + || \mathbf{z}_t ||_2 \le 2 b r + b \hspace*{10mm}
\end{align*}
Finally, Lemma \ref{thm:xball} below implies (A\ref{assBalls}) holds with $r_{\smallup} = r_{\smalldown} = r$.
\end{proof}

\bigskip

\begin{lem} \label{thm:xball}
For any orthogonal $N \times d$ matrix $\mathbf{U}$ and $\mathbf{p} \in \mathcal{S}$, condition (A\ref{assballS}) implies: 
$$\mathbf{U}^T ( \mathcal{S}) = \{\mathbf{x} \in \mathbb{R}^d : ||\mathbf{x}||_2 \le r \} \ \text{ and } \ \mathbf{U} \mathbf{U}^T (\mathbf{p}) \in \mathcal{S}$$
 \addcontentsline{toc}{subsubsection}{Lemma \ref{thm:xball} }
\end{lem}

\begin{proof}
Consider the orthogonal extension of $\mathbf{U}$, a matrix $\mathbf{W} = [\mathbf{U}, \widetilde{\mathbf{U}}] \in \mathbb{R}^{N \times N}$ formed by appending $N - d$ additional orthonormal columns to $\mathbf{U}$ that are also orthogonal to the columns of $\mathbf{U}$.  For any $\mathbf{p} \in \mathbb{R}^N$, we have: 
\begin{align*} 
 ||\mathbf{U} \mathbf{U}^T \mathbf{p} ||_2 & = || \mathbf{U}^T \mathbf{p} ||_2 \tag*{since orthogonality implies $\mathbf{U}$  is an isometry}  
 \\
 & \le || \mathbf{W}^T \mathbf{p} ||_2  \tag*{because $|| \mathbf{W}^T \mathbf{p} ||_2^2 = ||\mathbf{U}^T \mathbf{p} ||_2^2 + ||\widetilde{\mathbf{U}}^T \mathbf{p} ||_2^2$}
 \\
 & = || \mathbf{W} \mathbf{W}^T \mathbf{p} ||_2 \tag*{since $\mathbf{W}$ is also an isometry}
 \\
 & = || \mathbf{p} ||_2 \tag*{due to the fact that $\mathbf{W}^T = \mathbf{W}^{-1}$ as $\mathbf{W}$ is square and orthogonal}
\end{align*} 
Combined with (A\ref{assballS}), this implies $\mathbf{U} \mathbf{U}^T (\mathbf{p}) \in \mathcal{S}$ and $||\mathbf{x}||_2 \le r$  for any $\mathbf{x} \in \mathbf{U}^T ( \mathcal{S})$.  
Now fix arbitrary $\mathbf{x} \in \mathbb{R}^d$ which satisfies $||\mathbf{x}||_2 \le r$.  By orthogonality of $\mathbf{U}$: \\
\[||\mathbf{U} \mathbf{x}||_2 = ||\mathbf{x}||_2  \le r \implies \mathbf{U} \mathbf{x} \in \mathcal{S}, \text{ and } \mathbf{U}^T \mathbf{U} \mathbf{x} = \mathbf{x} \implies \mathbf{x} \in \mathbf{U}^T ( \mathcal{S}) 
\qedhere
\] 
\end{proof}

\subsection{Proof of Lemma \ref{lem:contract}}

\begin{customlem}{\ref{lem:contract}}
For any orthogonal matrix $\widehat{\mathbf{U}}$ and any ${\mathbf{x} \in  \widehat{\mathbf{U}}^T (\mathcal{S})}$, define ${\widehat{\mathbf{p}} = \widehat{\mathbf{U}} \mathbf{x} \in \mathbb{R}^N}$.
\\
Under condition (A\ref{assballS}): ${\widehat{\mathbf{p}} \in \mathcal{S}}$ and $\widehat{\mathbf{p}} = $ \textproc{FindPrice}(${\mathbf{x},  \widehat{\mathbf{U}}, \mathcal{S}, \mathbf{0}}$).   
\end{customlem}

\begin{proof}
Given $\mathbf{x} \in  \widehat{\mathbf{U}}^T (\mathcal{S})$, there exists $\mathbf{p} \in \mathcal{S}$ with $\widehat{\mathbf{U}}^T \mathbf{p} = \mathbf{x}$.
The proof of Lemma \ref{thm:xball} implies  ${||\widehat{\mathbf{p}}||_2 \le ||\mathbf{p}||_2}$ and  $\widehat{\mathbf{p}} = \widehat{\mathbf{U}} \mathbf{x} =  \widehat{\mathbf{U}} \widehat{\mathbf{U}}^T \mathbf{p} \in \mathcal{S}$ when this set is a centered Euclidean ball. Finally, we note that $\widehat{\mathbf{U}}^T \widehat{\mathbf{p}} = \mathbf{x}$ since $\widehat{\mathbf{U}}^T \widehat{\mathbf{U}} = \mathbf{I}_{d \times d}$, so $\widehat{\mathbf{p}}$ is the minimum-norm vector in $\mathcal{S}$ which is mapped to $\mathbf{x}$ by $\widehat{\mathbf{U}}^T$.
\end{proof}

\subsection{Proof of Theorem \ref{thm:modified}}

\begin{customthm}{\ref{thm:modified}}
Suppose $\text{span}(\widehat{\mathbf{U}}) = \text{span}(\mathbf{U})$, i.e.\ our orthogonal estimate has the same column-span as the underlying (rank $d$) latent  product-feature matrix.  Let  $\mathbf{p}_1, \dots, \mathbf{p}_T \in \mathcal{S}$ denote the prices selected by our modified \algfirstabbrev{} algorithm with $\widehat{\mathbf{U}}$ used in place of the underlying $\mathbf{U}$ and $\eta, \delta, \alpha$ chosen as in Theorem \ref{cor:knownu}.  Under conditions (A\ref{assbz})-(A\ref{assballS}), there exists $C > 0$ such that for any $\mathbf{p} \in \mathcal{S}$:
\begin{equation*} 
\mathbb{E}_{\bm{\epsilon}, \bm{\xi}} \left[ \sum_{t=1}^T R_t(\mathbf{p}_t) -  \sum_{t=1}^T R_t(\mathbf{p})  \right] 
\le   C  br (r+1)  T^{3/4}  d^{1/2} 
\end{equation*}
\end{customthm}

\begin{proof}
Define $\displaystyle \widebar{\mathbf{p}} = \argmin_{\mathbf{p} \in \mathcal{S}} \mathbb{E}_{\bm{\epsilon}} \sum_{t=1}^T R_t(\mathbf{p})$, ${\mathbf{p}^*} = \mathbf{U} \mathbf{U}^T \widebar{\mathbf{p}}$.  
Note that  $\mathbb{E}_{\bm{\epsilon}} \left[ \sum_{t=1}^T R_t({\mathbf{p}^*}) \right] =\mathbb{E}_{\bm{\epsilon}}\left[ \sum_{t=1}^T R_t(\widebar{\mathbf{p}}) \right]$ and ${\mathbf{p}^* \in \mathcal{S}}$ by Lemma \ref{thm:xball}, so ${\mathbf{p}^*}$ is an equivalently optimal setting of the product prices.
  Since $\mathbf{U}$ and $\widehat{\mathbf{U}}$ share the same column-span, there exists low-dimensional action $\mathbf{x}^* \in \mathbb{R}^k$ such that ${{\mathbf{p}^*} = \widehat{\mathbf{U}} {\mathbf{x}^*}}$.
By orthogonality of $\widehat{\mathbf{U}}$: $\widehat{\mathbf{U}}^T \widetilde{\mathbf{p}} =  \widehat{\mathbf{U}}^T  \widehat{\mathbf{U}} \mathbf{x}^* = \mathbf{x}^*$, so $\mathbf{x}^* \in \widehat{\mathbf{U}}^T (\mathcal{S})$ is a feasible solution to our modified \algfirstabbrev{} algorithm.
For $\mathbf{x} \in \mathbb{R}^d$ and $\mathbf{p} = \widehat{\mathbf{U}} \mathbf{x} \in \mathbb{R}^N$, we re-express the expected revenue at this price vector by introducing $f_{t,\widehat{\mathbf{U}}}$ as a function of $\mathbf{x}$ parameterized by $\widehat{\mathbf{U}}$, as similarly done in (\ref{eq:lowdimensionalequivalence}):  
\begin{align*} 
f_{t,\widehat{\mathbf{U}}}(\mathbf{x})  = \mathbb{E}_{\bm{\epsilon}}[R_t(\mathbf{p})] 
=  \mathbf{x}^T \widehat{\mathbf{U}}^T \mathbf{U}  \mathbf{V}_t \mathbf{U}^T  \widehat{\mathbf{U}}
  \mathbf{x} - \mathbf{x}^T \widehat{\mathbf{U}}^T \mathbf{U} \mathbf{z}_t
\numberthis \label{eq:fy}
\end{align*}
Convexity of $R_t$ in $\mathbf{p}$ implies $f_{t,\widehat{\mathbf{U}}}$ is convex in $\mathbf{x}$ for any $\widehat{\mathbf{U}}$.  Note that our modified \algfirstabbrev{} algorithm is (in expectation) running online projected gradient descent on a smoothed version of each $f_{t,\widehat{\mathbf{U}}}$, defined similarly as in (\ref{eq:smoothed}).  Via the same argument employed in the previous section (based on Theorem \ref{thm:flax} and the proof of Theorem \ref{cor:knownu}), we can show that for  $\mathbf{x}^* \in \widehat{\mathbf{U}}^T (\mathcal{S})$:
\begin{equation*}
\mathbb{E}_{\bm{\xi}} \left[ \sum_{t=1}^T f_{t,\widehat{\mathbf{U}}}(\widetilde{\mathbf{x}}_t) - \sum_{t=1}^T  f_{t,\widehat{\mathbf{U}}}(\mathbf{x}^*) \right]   \le C  br (r+1)  T^{3/4}  d^{1/2} 
\end{equation*}
where $\widetilde{\mathbf{x}}_t$ are the low-dimensional actions chosen in Step 5 of our modified \algfirstabbrev{} algorithm, such that $\mathbf{p}_t = \widehat{\mathbf{U}} \widetilde{\mathbf{x}}_t$ for the prices output by this method.  To conclude the proof, we recall that for the \algfirstabbrev{}-selected $\mathbf{p}_t$:
\begin{align*}
& \hspace*{-2mm} \mathbb{E} \hspace*{-0.5mm} \sum_{t=1}^T \hspace*{-0.5mm}  R_t(\mathbf{p}_t)  \hspace*{-0.5mm}  = \hspace*{-0.5mm}   \mathbb{E} \hspace*{-0.5mm}   \sum_{t=1}^T \hspace*{-0.5mm}  f_{\hspace*{-0.5mm}  t,\widehat{\mathbf{U}}}\hspace*{-0.5mm}  (\widetilde{\mathbf{x}}_t), \ 
 \mathbb{E} \hspace*{-0.5mm}  \sum_{t=1}^T \hspace*{-0.5mm}  R_t(\mathbf{p}^*) \hspace*{-0.5mm}  = \hspace*{-0.5mm} \sum_{t=1}^T \hspace*{-0.5mm}  f_{\hspace*{-0.5mm} t,\widehat{\mathbf{U}}} \hspace*{-0.5mm} (\mathbf{x}^*)  \hspace*{1mm} 
\qedhere
\end{align*}

\end{proof}

\subsection{Proof of Lemma \ref{lem:zerochance}}

\begin{customlem}{\ref{lem:zerochance}}
Suppose that for $t = 1, \dots, T$: $\bm{\epsilon}_t = 0$ and $\mathbf{V}_t \succ 0$.  If each $\mathbf{p}_t$ is independently uniformly distributed within some (uncentered) Euclidean ball of strictly positive radius, then span$(\mathbf{q}_1,\dots, \mathbf{q}_d) = \text{span}(\mathbf{U})$ almost surely.  
\end{customlem}

\begin{proof}
In Lemma \ref{lem:zerochance}, we suppose that each $\mathbf{p}_t = \widetilde{\mathbf{p}}_t + \bm{\zeta}_t$, where each $\bm{\zeta}_t$ is uniformly drawn from a centered Euclidean ball of nonzero radius in $\mathbb{R}^N$ and $\mathbf{z}_t, \mathbf{V}_t, \widetilde{\mathbf{p}}_t$ are fixed independently of the randomness in $\bm{\zeta}_t$.  Note that each ${\mathbf{q}_t = \mathbf{U} \mathbf{s}_t}$ where ${\mathbf{s}_t = \mathbf{z}_t - \mathbf{V}_t \mathbf{U}^T \mathbf{p}_t \in \mathbb{R}^d}$.
 Thus, ${\text{span}(\mathbf{q}_1,\dots, \mathbf{q}_d) \subseteq \text{span}(\mathbf{U})}$ and the two spans must be equal if $\mathbf{s}_1, \dots, \mathbf{s}_d$ are linearly independent.  
 
 To show linear independence holds almost surely, we proceed inductively by proving ${\Pr (\mathbf{s}_{t} \in \text{span}( \mathbf{s}_1, \dots, \mathbf{s}_{t-1}) ) = 0}$ for any $1 < t \le d$.  We first note that $\mathbf{s}_{t} = \mathbf{z}_t - \mathbf{V}_t \mathbf{U}^T \widetilde{\mathbf{p}}_t  - \mathbf{V}_t \mathbf{U}^T \bm{\zeta}_t$.  Since $\mathbf{V}_t \succ 0$ is invertible and $\mathbf{U}$ is orthogonal, $\mathbf{V}_t \mathbf{U}^T \bm{\zeta}_t$ is uniformly distributed over a nondegenerate ellipsoid $\mathcal{E} \subset \mathbb{R}^d$ with nonzero variance under any projection in $\mathbb{R}^d$.  Since this includes directions orthogonal to the   $(t-1)$-dimensional subspace spanned by  $\mathbf{s}_{1} + \mathbf{V}_1 \mathbf{U}^T \widetilde{\mathbf{p}}_1  -  \mathbf{z}_1, \dots, \mathbf{s}_{t-1} + \mathbf{V}_{t-1} \mathbf{U}^T \widetilde{\mathbf{p}}_{t-1}  -  \mathbf{z}_{t-1}$, this subspace has measure zero under the uniform distribution over $\mathcal{E}$ (for $t \le d$).
\end{proof}

\begin{thm}[\citeauthor{Yu15}, \citeyear{Yu15}] \label{thm:yu}
Let ${\sigma_1 > \dots > \sigma_d > 0}$ denote the nonzero singular values of rank $d$ matrix ${\mathbf{Q} \in \mathbb{R}^{N \times d}}$, whose left singular vectors are represented as columns in matrix $\mathbf{U} \in \mathbb{R}^{N \times d}$ (such that $\mathbf{Q}$ has SVD: $\mathbf{U} \mathbf{\Sigma} \mathbf{V}^T$).  If $\widehat{\mathbf{U}} \in \mathbb{R}^{N \times d}$ similarly contains the left singular vectors of some other $N \times d$ matrix $\widehat{\mathbf{Q}}$, then there exists orthogonal matrix $\widehat{\mathbf{O}} \in \mathbb{R}^{d \times d} $ such that 
\[
|| \widehat{\mathbf{U}} \widehat{\mathbf{O}}  - \mathbf{U} ||_F \le \frac{2\sqrt{2d}}{\sigma^2_d} \big(2 \sigma_1 + || \widehat{\mathbf{Q}} - \mathbf{Q} ||_{\textup{op}} \big) || \widehat{\mathbf{Q}} - \mathbf{Q} ||_{\textup{op}}
\]
 \addcontentsline{toc}{subsubsection}{Theorem \ref{thm:yu}}
\end{thm}

\subsection{Proof of Theorem \ref{thm:unknownspan}}
\label{sec:proof:unknownspan}

\begin{customthm}{\ref{thm:unknownspan}}
For unknown $\mathbf{U}$, let $\mathbf{p}_1, \dots, \mathbf{p}_T$ be the prices selected by the \algsecondabbrev{} algorithm with $\eta, \delta, \alpha$ set as in Theorem \ref{cor:knownu}.  Suppose $\bm{\epsilon}_t$ follows a sub-Gaussian$(\sigma^2 )$ distribution and has statistically independent dimensions for each $t$.   If  (A\ref{assbz})-(A\ref{assballS}) hold, then there exists ${C > 0}$ such that for any $\mathbf{p} \in \mathcal{S}$:
\begin{equation*}
\mathbb{E}_{\bm{\epsilon}, \bm{\xi}} \left[ \sum_{t=1}^T R_t(\mathbf{p}_t) -  \sum_{t=1}^T R_t(\mathbf{p})  \right] 
\le C Q rb (4r + 1) d T^{3/4} 
\end{equation*}
Here, $Q = \max\left\{1, \sigma^2 \left(\frac{2\sigma_1 + 1}{\sigma_d^2}\right) \right\}$ with $\sigma_1$ (and $\sigma_d$) defined as the largest (and smallest) nonzero singular values of the underlying rank $d$ matrix $\widebar{\mathbf{Q}}^*$ defined in (\ref{eq:qbar}).
\end{customthm}

\begin{proof}
For notational convenience, suppose that $T$ is divisible by $d$, $T^{3/4} \ge d \ge 3$, and the noise-variation parameter $\sigma \ge 1$ throughout our proof.  Throughout, the unknown $\mathbf{U}$ is orthogonal and rank $d$, and we let $ \mathbf{p}^* = \argmin_{\mathbf{p} \in \mathcal{S}} \mathbb{E} \left[\sum_{t=1}^T R_t(\mathbf{p})\right]$ denote the optimal product pricing.

Recall from the proof of Theorem \ref{thm:modified} that under our low-rank demand model, we can redefine $\mathbf{p}^* \leftarrow \mathbf{U} \mathbf{U}^T \mathbf{p}^* \in \mathcal{S}$ and still ensure $ \mathbf{p}^* = \argmin_{\mathbf{p} \in \mathcal{S}} \mathbb{E}\left[ \sum_{t=1}^T  R_t(\mathbf{p}) \right]$.  Thus, we suppose without loss of generality that the optimal prices can be expressed as $\mathbf{p}^* = \mathbf{U} \mathbf{x}^*$
for some corresponding low-dimensional action $\mathbf{x}^* \in \mathbf{U}^T (\mathcal{S})$.

For additional clarity, we use $\widehat{\mathbf{U}}_t$ to denote the current $N \times d$ estimate of the underlying product features obtained in Step 10 of our \algsecondabbrev{} algorithm at round $t$.  Note that the $\widehat{\mathbf{U}}_t$ are random variables which are determined by both the noise in the observed demands and the randomness employed within our pricing algorithm.  Letting $\mathbf{p}_t = \widehat{\mathbf{U}}_t \mathbf{x}_t$ denote the prices chosen by the \algsecondabbrev{} algorithm in each round (and $\mathbf{x}_t \in \widehat{\mathbf{U}}_t^T (\mathcal{S}) $ the corresponding low-dimensional actions), we have: 
\begin{align*}
& \mathbb{E} \hspace*{-0.1mm}  \sum_{t=1}^T \hspace*{-0.4mm}  \left[ R_t( \mathbf{p}_t) - R_t( \mathbf{p}^*) \hspace*{-0.5mm}  \right] \hspace*{-0.5mm}  =   \numberthis \label{eq:regretunk}
\\
& \mathbb{E} \hspace*{-0.1mm}  \sum_{t=1}^{T^{3/4}} \hspace*{-0.4mm}  \left[   f_{t,\widehat{\mathbf{U}}_t}( \mathbf{x}_t) - f_{t,\mathbf{U}}( \mathbf{x}^*) \hspace*{-0.5mm} \right] 
\hspace*{-0.5mm}  + 
 \mathbb{E} \hspace*{-2mm}  \sum_{t=T^{3/4}}^{T} \hspace*{-1mm} \left[    f_{t,\widehat{\mathbf{U}}_t}( \mathbf{x}_t)  - f_{t,\widehat{\mathbf{U}}_t}( \widetilde{\mathbf{x}}) \hspace*{-0.5mm}  \right]
 \hspace*{-0.5mm}  + 
 \mathbb{E} \hspace*{-2mm}   \sum_{t=T^{3/4}}^{T} \hspace*{-1mm}  \left[  f_{t,\widehat{\mathbf{U}}_t}( \widetilde{\mathbf{x}})  - f_{t,\mathbf{U}}(\mathbf{x}^*) \hspace*{-0.5mm}  \right]
\\ 
& \text{ where $f_{t,\mathbf{U}}$ is defined as in (\ref{eq:fy}) and we define }  \widetilde{\mathbf{x}} = \argmin_{\mathbf{x} \in \mathbf{U}^T (\mathcal{S})} \ \mathbb{E}\left[ \sum_{t=T^{3/4}}^T  f_{t,\widehat{\mathbf{U}}_t}( \mathbf{x}) \right].
\end{align*}
The proof of Theorem \ref{cor:knownu} ensures both $|f_{t,\mathbf{U}}|$ and $|f_{t,\widehat{\mathbf{U}}_t}|$  (for any orthogonal $\widehat{\mathbf{U}}_t$) are bounded by $rb(1+r)$ over all $\mathbf{x} \in \mathbf{U}^T (\mathcal{S})$, so we can trivially bound the first summand in (\ref{eq:regretunk}): 
\[
 \sum_{t=1}^{T^{3/4}} \left[ f_{t,\widehat{\mathbf{U}}_t}( \mathbf{x}_t) - f_{t,\mathbf{U}}( \mathbf{x}^*) \right] \le  rb(1+r) \cdot T^{3/4}
\]
To bound the second summand in (\ref{eq:regretunk}), we first point out that $\mathbf{U}^T (\mathcal{S}) = \widehat{\mathbf{U}}_t^T (\mathcal{S})$ by Lemma \ref{thm:xball} (since all $\widehat{\mathbf{U}}_t$ are restricted to be orthogonal).  Thus, Algorithm \ref{alg:second} is essentially running the classic gradient-free bandit method of 
\cite{Flaxman05} to optimize the functions $f_{t,\widehat{\mathbf{U}}_t}$ over the low-dimensional action-space $\mathbf{U}^T (\mathcal{S})$, and the second term is exactly the regret of this method stated in Theorem \ref{cor:knownu}:
\[   \mathbb{E} \hspace*{-1mm} \sum_{t=T^{3/4}}^{T} \left[  f_{t,\widehat{\mathbf{U}}_t}( \mathbf{x}_t)  - f_{t,\widehat{\mathbf{U}}_t}( \widetilde{\mathbf{x}})  \right] \le C  br (r+1)  \left[T - T^{3/4}\right]^{3/4}  d^{1/2} 
\]
Finally, we complete the proof by bounding the third summand in (\ref{eq:regretunk}).  Defining ${\mathcal{O} \subset \mathbb{R}^{d \times d}}$ as the set of orthogonal $d \times d$ matrices, we have: 
\begin{align*}
& \mathbb{E} \hspace*{-1mm}  \sum_{t=T^{3/4}}^{T} \left[  f_{t,\widehat{\mathbf{U}}_t}( \widetilde{\mathbf{x}})  - f_{t,\mathbf{U}}(\mathbf{x}^*)  \right]
 \le
\inf_{\mathbf{O} \in \mathcal{O}} \ \mathbb{E}   \sum_{t=T^{3/4}}^{T} \left[  f_{t,\widehat{\mathbf{U}}_t}(  \mathbf{O} \mathbf{x}^*) - f_{t,\mathbf{U}}(\mathbf{x}^*)  \right]
\\
& \tag*{since $\mathbf{x}^* \in \mathbf{U}^T (\mathcal{S}) \implies \mathbf{O} \mathbf{x}^* \in \mathbf{U}^T (\mathcal{S})$ by Lemma \ref{thm:xball}, and $\widetilde{\mathbf{x}}$ is an argmin over $ \mathbf{U}^T (\mathcal{S})$} 
\\
\le & \inf_{\mathbf{O} \in \mathcal{O}} \ (T-T^{3/4}) \cdot \mathbb{E}  \left[  f_{t,\widehat{\mathbf{U}}_t}( \mathbf{O} \mathbf{x}^*)  - f_{t,\mathbf{U}}(\mathbf{x}^*) \right] 
\\
& \tag*{where we've fixed $\displaystyle t = \argmax_{t' \in [T^{3/4},T]} \ \inf_{\mathbf{O} \in \mathcal{O}} \ \mathbb{E}  \left[  f_{t',\widehat{\mathbf{U}}_{t'}}( \mathbf{O} \mathbf{x}^*)  - f_{t',\mathbf{U}}(\mathbf{x}^*) \right]$}
\\
\le  & 
(T-T^{3/4}) \cdot \mathbb{E}  \left[  f_{t,\widehat{\mathbf{U}}_t}( \widehat{\mathbf{O}} \mathbf{x}^*)  - f_{t,\mathbf{U}}(\mathbf{x}^*) \right] 
\end{align*}
where now choose $\widehat{\mathbf{O}} \in \mathcal{O}$ as the orthogonal matrix such that  $\mathbb{E}  || \widehat{\mathbf{U}}_t \widehat{\mathbf{O}}  - \mathbf{U} ||_F$ satisfies the bound of Lemma \ref{thm:uestimate} for the $t \ge T^{3/4}$  fixed above.  
Defining ${\bm{\Delta} = \mathbf{U} \mathbf{x}^* - \widehat{\mathbf{U}}_t   \widehat{\mathbf{O}} \mathbf{x}^* \in \mathbb{R}^d}$, we plug in the definition of $f_{t,\widehat{\mathbf{U}}}$ from (\ref{eq:fy}) and simplify to obtain the following bound:
\begin{align*}
& \mathbb{E}  \left[  f_{t,\widehat{\mathbf{U}}_t}( \widehat{\mathbf{O}} \mathbf{x}^*)  - f_{t,\mathbf{U}}(\mathbf{x}^*) \right]
\\  
\le &  \mathbb{E}  \left[ ||\bm{\Delta}||_2^2 || \mathbf{U} ||_\text{op} || \mathbf{V}_t ||_\text{op} || \mathbf{U}^T ||_\text{op}  + 2 ||\bm{\Delta}||_2 ||\mathbf{x}^* ||_2 ||\mathbf{V}_t ||_\text{op} || \mathbf{U}^T ||_\text{op}  + ||\bm{\Delta}||_2  || \mathbf{z}_t ||_2  || \mathbf{U} ||_\text{op}  \right]
\\ 
\le &  \mathbb{E}  \left[ ||\bm{\Delta}||_2^2 || \mathbf{V}_t ||_\text{op} + 2 ||\bm{\Delta}||_2 ||\mathbf{x}^* ||_2 ||\mathbf{V}_t  ||_\text{op}  + ||\bm{\Delta}||_2  || \mathbf{z}_t ||_2  \right]
\\
\le &  \left( 4 ||\mathbf{x}^* ||_2 \mathbf{V}_t ||_\text{op}  +  || \mathbf{z}_t ||_2  \right)  \cdot  \mathbb{E}  \left[ ||\bm{\Delta}||_2 \right]   
\\
\tag*{since  $||\bm{\Delta}||_2  \le (|| \mathbf{U} ||_\text{op} +|| \widehat{\mathbf{U}}_t   \widehat{\mathbf{O}}  ||_\text{op}) || \mathbf{x}^* ||_2 \le 2  || \mathbf{x}^* ||_2  $ by orthogonality of $\widehat{\mathbf{O}}, \widehat{\mathbf{U}}_t, \mathbf{U}$}
\\
\le & C rb \left( 4 r  +  1 \right)  \left[T^{3/4} \right]^{-1/2}   d \sigma^2 \left( \frac{2 \sigma_1 + 1}{\sigma^2_d} \right)    \tag*{under (A\ref{assbz})-(A\ref{assbV})}
\\
& \text{since }  \mathbb{E}  \left[ ||\bm{\Delta}||_2 \right] \le || \mathbf{x}^* ||_2  \cdot \mathbb{E}  \left[ || \widehat{\mathbf{U}}_t \widehat{\mathbf{O}}  - \mathbf{U} ||_F \right]
\le 
C \left[T^{3/4}\right]^{-1/2}   d \sigma^2 \left( \frac{2 \sigma_1 + 1}{\sigma^2_d} \right) || \mathbf{x}^* ||_2 
\\
& \text{by Lemma \ref{thm:uestimate} (recalling that we fixed $t \ge T^{3/4}$).} 
\\
\end{align*}  
Combining our bounds for each of the three summands in (\ref{eq:regretunk}) yields the following upper bound for the left-hand side, from which the inequality presented in Theorem \ref{thm:unknownspan} can be derived:
\[
C rb \left[(1 + r) T^{3/4} + (1 +r ) d^{1/2} \left(T - T^{3/4}\right)^{3/4}  + (4r + 1)d \sigma^2 \left(\frac{2\sigma_1 + 1}{\sigma_d^2}\right) \left(T^{5/8} - T^{3/8} \right) \right]   
\qedhere
\]
\end{proof}

\begin{lem}  \label{thm:uestimate}  For the $\widehat{\mathbf{U}}$ produced in Step 10 of the \algsecondabbrev{} algorithm after $T$ rounds and any feasible low-dimensional action $\mathbf{x} \in \widehat{\mathbf{U}}^T (\mathcal{S})$, there exists orthogonal $d \times d$ matrix $\widehat{\mathbf{O}}$ and universal constant $C$ such that:

\[ \mathbb{E} \left[ || \widehat{\mathbf{U}} \widehat{\mathbf{O}}  - \mathbf{U} ||_F \right] 
\le  C T^{-1/2} d \sigma^2 \left( \frac{2 \sigma_1 + 1}{\sigma^2_d} \right)
\]
where $\sigma_1$ and $\sigma_d$ denote the largest and smallest singular values of the underlying matrix $\widebar{\mathbf{Q}}^*$ defined in (\ref{eq:qbar}).

\begin{proof}
Our proof relies on standard random matrix concentration results presented in Lemma \ref{thm:subgaus} and the variant of the Davis-Kahan  theory proposed by  \citetsi{Yu15}, which is restated here as Theorem \ref{thm:yu}.

\begin{lem}[variant of Lemma 4.2 in \citetsi{Rigollet15}] \label{thm:subgaus}
Let $\mathbf{E}$ be a $N \times d$ matrix (with $N \ge d$) of i.i.d.\ entries drawn from a sub-Gaussian$(\sigma^2 )$ distribution.
Then, with probability $1 - \delta$:
\[ ||\mathbf{E}||_\textup{op} \le 2\sigma \left[2\sqrt{N \log(12)} + \sqrt{2 \log(1/\delta) } \right]
\]
\end{lem}

Recall that random variable $X$ follows sub-Gaussian($\sigma^2$) distribution if ${ \mathbb{E}[X] = 0 }$  and 
${\Pr(|X| > x) \le  2 \exp(- \frac{x^2}{2 \sigma^2})}$ for all $x > 0$, and random vector $\mathbf{w} \sim $ sub-Gaussian($\sigma^2$) if ${ \mathbb{E}[\mathbf{w}] = 0 }$ and $\mathbf{u}^T \mathbf{w}$ is a sub-Gaussian($\sigma^2$) random variable for any unit vector $\mathbf{u}$.  Since the components of $\bm{\epsilon}_t$ are presumed statistically i.i.d., each value in $\widebar{\mathbf{E}} = \widehat{\mathbf{Q}} - \widebar{\mathbf{Q}}^*$ must be the mean of $T/d$ sub-Gaussian$(\sigma^2 / N)$ samples as a result of the averaging performed in Step 9 of our \algsecondabbrev{} algorithm.  Thus, the entries of $\widebar{\mathbf{E}}$ are distributed as sub-Gaussian$\left(\frac{\sigma^2d}{NT}\right)$.
Lemma \ref{thm:subgaus} implies:
\begin{align*}
 \mathbb{E} ||\widebar{\mathbf{E}}||_\text{op} & = \int_{x=0}^\infty \Pr( ||\widebar{\mathbf{E}}||_\text{op} > x) \ \mathrm{d}x \\
& \le  \int_{x=0}^\infty \exp \left( - \frac{1}{2} \left(\sqrt{\frac{T}{d}}\frac{x}{2\sigma} - 2\sqrt{\log 12} \right)^2 \right)  \ \mathrm{d}x     \\
& = 2 \sigma \sqrt{\frac{\pi d}{2 T}} \left[1 + \text{erf}\left(\sqrt{2 \log 12}\right) \right]  
 \le 4 \sigma \sqrt{\frac{\pi d}{2T}} 
\\ 
 \mathbb{E}||\widebar{\mathbf{E}}||_\text{op}^2 & = 2  \int_{x=0}^\infty x \cdot \Pr( ||\widebar{\mathbf{E}}||_\text{op} > x) \ \mathrm{d}x \\
 & \le 2 \int_{x=0}^\infty x \cdot \exp \left( - \frac{1}{2} \left(\sqrt{\frac{T}{d}} \frac{x}{2\sigma} - 2\sqrt{\log 12} \right)^2 \right)  \ \mathrm{d}x     \\
 & = \frac{8 \sigma^2d}{T} \left[ \sqrt{2\pi \log 12} +  \sqrt{2 \pi \log 12} \cdot \text{erf}(\sqrt{2 \log 12}) + \frac{1}{144}   \right]
 \\ & \le \frac{24 \sigma^2 d}{T} \sqrt{2\pi \log 12}
\end{align*}
When $T \ge d, \sigma \ge 1$, both $ \mathbb{E} ||\widebar{\mathbf{E}}||_\text{op}$ and $\mathbb{E}||\widebar{\mathbf{E}}||_\text{op}^2$ are upper-bounded by $24  \sigma^2 \sqrt{6 \pi d/T}$.  Combining Theorem \ref{thm:yu}  with these concentration bounds implies that there exists $d \times d$ orthogonal matrix $\widehat{\mathbf{O}}$ such that:
\begin{align*}
\mathbb{E} \left[ || \widehat{\mathbf{U}} \widehat{\mathbf{O}}  - \mathbf{U} ||_F \right] & 
\le \frac{2\sqrt{2d}}{\sigma^2_d} \left(2 \sigma_1  \mathbb{E} \left[ || \widehat{\mathbf{Q}} - \widebar{\mathbf{Q}}^* ||_{\text{op}} \right] + \mathbb{E} \left[ || \widehat{\mathbf{Q}} - \widebar{\mathbf{Q}}^* ||_{\text{op}}^2 \right] \right)
\\
& \le  96 \sqrt{\frac{3 \pi}{T}} d \sigma^2 \left( \frac{2 \sigma_1 + 1}{\sigma^2_d} \right) \qedhere
\end{align*}
\end{proof}
\end{lem}

\clearpage
\section{Pricing against an Imprecise Adversary}
\label{sec:imprecise}

Theorem \ref{thm:qbound} below  illustrates a basic scenario under which an explicit high-probability bound for the constant $Q$ from Theorem \ref{thm:unknownspan} can be obtained.  Throughout our subsequent discussion, the largest and smallest nonzero singular values of a rank-$d$ matrix $\mathbf{A}$ will be denoted as $\sigma_1(\mathbf{A})$ and $\sigma_d(\mathbf{A})$, respectively.  We now assume that the adversary can only coarsely control the underlying baseline demand parameters $\mathbf{z}_t$ in (\ref{eq:lowrank}).  More specifically, we suppose that in each round:  $\mathbf{z}_t = \mathbf{z}_t' + \bm{\gamma}_t$, 
where only $\mathbf{z}_t'$ (and $\mathbf{V}_t$) may be adversarially selected and the $\bm{\gamma}_t$ are purely stochastic terms outside of the adversary's control.  In this scenario, we presume  a random $d \times d$  noise matrix $\mathbf{\Gamma}$ is drawn before the initial round such that:
\begin{itemize}  \setlength\itemsep{0em}
\refstepcounter{assumptioncount} \label{assgamma}
{\setlength\itemindent{5pt} \item[(A\arabic{assumptioncount})] 
Each entry $\mathbf{\Gamma}_{i,j}$ is independently sampled with mean zero and magnitude  bounded almost surely by $b/2$ \  (i.e. $\mathbb{E}[\mathbf{\Gamma}_{i,j}] = 0, \  |\mathbf{\Gamma}_{i,j}| \le b/2$ for all $i,j$). 
} 
\end{itemize}
  
Recall that the constant $b > 0$ upper bounds the magnitude of each $\mathbf{z}_t$ as specified in (A\ref{assbz}).  
Once the values of $\mathbf{\Gamma}$ have been sampled, we suppose that in round $t$: $\bm{\gamma}_t =  \mathbf{\Gamma}_{*,j}$ is simply taken to be the $j$th column of this matrix with $j = 1 +  (t-1) \mod d$ (traversing the columns of $\mathbf{\Gamma}$ in order).    
\noindent Since boundedness of the values in $\mathbf{\Gamma}$ implies these entries follow a sub-Gaussian$(b^2/4)$ distribution, the following result applies:

\begin{lem}[variant of Theorem 1.2 in \citetsi{Rudelson08}]  With probability at least $1 - C_b \epsilon - {c_b}^d$:
$$  \sigma_d(\mathbf{\Gamma}) \ge \epsilon  / \sqrt{d}
$$
where $C_b > 0$ and $c_b \in (0,1)$ are constants that depend (polynomially) only on $b$. 
\label{lem:rudelson}
\end{lem}

\noindent In selecting $\mathbf{z}_t', \mathbf{V}_t$, we assume the imprecise adversary is additionally restricted to ensure:  
\begin{itemize}  \setlength\itemsep{0em}
\refstepcounter{assumptioncount} \label{assZVbound}
{\setlength\itemindent{5pt} \item [(A\arabic{assumptioncount})] 
There exists  $\displaystyle s < \frac{1 - {c_b}^d}{C_b d b} < 1$ such that for all $t$:   \  $ 
\displaystyle ||\mathbf{z}_t'||_2 + r \cdot ||\mathbf{V}_t||_{op} \le s  \cdot \min_{1 \le j \le d} ||\mathbf{\Gamma}_{*,j}||_2$.
}
\end{itemize}
where constants $c_b, C_b$ are given by Lemma \ref{lem:rudelson} (see \citetsi{Rudelson08} for details),  and $r \ge 1$ is still used to denote the radius of the set of feasible prices $\mathcal{S}$.
Note that these additional assumptions do not conflict with condition (A\ref{assbz}) required in Theorem 4, since (A\ref{assgamma}), (A\ref{assZVbound}) together ensure that $|| \mathbf{z}_t ||_2 \le b$ for $\mathbf{z}_t = \mathbf{z}_t' + \bm{\gamma}_t$.  With these assumptions in place, we now provide an explicit bound for the constant $Q$ defined in Theorem \ref{thm:unknownspan}.

\begin{thm}
Under this setting of an imprecise adversary where conditions (A10) and (A11) are met,  for any ${\tau \in (\frac{1}{2}  C_b s  b  d + {c_b}^d , \ 1)}$, 
Theorem 4 holds with:
$$Q  \le \frac{ 2 \sigma d C_b   (2b + 1)}{ 2(\tau - {c_b}^d) - C_b s b d } $$
 with probability $\ge  1 - \tau$ (over the initial random sampling of $\mathbf{\Gamma}$).
\label{thm:qbound}
 \addcontentsline{toc}{subsubsection}{Theorem \ref{thm:qbound}}
\end{thm}
\begin{proof}
Recall that $\sigma_1$ (and $\sigma_d$) denote the largest (and smallest) nonzero singular values of the underlying rank $d$ matrix $\widebar{\mathbf{Q}}^*$ defined in (\ref{eq:qbar}).
For suitable constants $c_1, c_2$: we show that  $\sigma_1 \le c_1$ and $\sigma_d \ge c_2$ with high probability, which then implies the upper bound: ${Q \le \max\{1, \frac{\sigma}{c_2^2}(2c_1 + 1) \}}$. 
We first note that the orthogonality of $\mathbf{U}$ implies $\widebar{\mathbf{Q}}^* = \mathbf{U} \widebar{\mathbf{S}}$ has the same nonzero singular values as the square matrix $\widebar{\mathbf{S}}$, whose $j$th column is given by: 
\begin{equation}
\widebar{\mathbf{S}}_{*,j} =  \frac{d}{T} \sum_{i=1}^{T/d} \Big[ \mathbf{z}'_{j + d(i-1)} + \bm{\gamma}_{j + (i-1)d} -  \mathbf{V}_{j + (i-1)d} \mathbf{U}^T \mathbf{p}_{j + (i-1)d}  \Big] 
\end{equation}
As $\widebar{\mathbf{S}}$ has $d$ columns, we have:  
$$ \sigma_1( \widebar{\mathbf{Q}}^*) = \sigma_1(\widebar{\mathbf{S}})  \le \sqrt{d} \cdot \max_j ||\widebar{\mathbf{Q}}_{*,j}||_2  \le \frac{b(1 + s)\sqrt{d}}{2} 
$$
where the latter inequality derives from the fact that (A\ref{assballS}) and orthogonality of $\mathbf{U}$ imply:
 \begin{align*}
 ||\widebar{\mathbf{S}}_{*,j}||_2 & \le \frac{d}{T} \sum_{i = 1}^{T/d} \Big[ || \bm{\gamma}_{j + (i-1)d} ||_2  + || \mathbf{z}'_{j + (i-1)d} ||_2 + r || \mathbf{V}_{j + (i-1)d} ||_{\text{op}} \Big]
\\
 & \le (1 + s) \cdot ||\mathbf{\Gamma}_{*,j}||_2 \le \frac{b}{2}(1 + s)  \tag*{by conditions (A\ref{assgamma}), (A\ref{assZVbound}) }
\end{align*} 
Via similar reasoning, we also obtain the bound:
\begin{align*}
 \sigma_1(\widebar{\mathbf{S}} - \mathbf{\Gamma}) \le \frac{s b \sqrt{d}}{2}  \numberthis \label{eq:upperqdiff}
\end{align*}
Subsequently, we invoke Lemma \ref{lem:rudelson}, which implies that with probability  $1 - \tau$: 
\begin{equation}
\sigma_d(\mathbf{\Gamma}) \ge \frac{\tau - {c_b}^d}{C_b \sqrt{d} } 
\label{eq:boundgamma}
\end{equation}
Combining  (\ref{eq:upperqdiff}) and (\ref{eq:boundgamma}), we obtain a high probability lower bound for $\sigma_d$ via the additive Weyl inequality (cf.\ Theorem 3.3.16 in \citetsi{Horn91}):
\begin{align*} \sigma_d(\widebar{\mathbf{Q}}^*) = \sigma_d(\widebar{\mathbf{S}})   & \ge \sigma_d(\mathbf{\Gamma}) - \sigma_1(\widebar{\mathbf{S}} - \mathbf{\Gamma})   
\ge  \frac{\tau - {c_b}^d}{C_b \sqrt{d} }  -  \frac{s b \sqrt{d}}{2}  
\ \text{ with probability } \ge 1 - \tau
\end{align*} 
The proof is completed by defining $\displaystyle c_1 = \frac{b(1 + s)\sqrt{d}}{2}, \ c_2 = \frac{\tau - {c_b}^d}{C_b \sqrt{d} }  -  \frac{s b \sqrt{d}}{2} $, and subsequent simplification of the resulting bound using the fact that $d \ge 1$ and $ s < 1$.
 \end{proof}
 
 \clearpage 
 \section{Additional Experimental Results}
 \label{sec:moreresults}  
 
  \subsection{Misspecified Demand Models}
  \label{sec:misresults}  

Beyond evaluating our pricing strategies in settings where underlying demand curves adhere to our low-rank model in (\ref{eq:lowrank}), we now consider different environments where our assumptions are purposefully violated, in order to investigate robustness and how well each approach generalizes to other types of demand behavior.  As our interest lies in high-dimensional pricing applications, the number of products is fixed to $N = 100$ throughout this section.  Once again,    
$\mathbf{p}_t$ and $\mathbf{q}_t$ are presumed to represent suitably rescaled prices/aggregate-demands, such that the set of feasible prices $\mathcal{S}$ can always be fixed as a centered sphere of radius $r = 20$.
Although none of the demand models considered here possesses explicit low-rank structure, we nevertheless apply our \algsecondabbrev{} pricing algorithm with various choices of the rank parameter $ 1 \le d \le N = 100$.

\textbf{Linear full-rank model.} We first study a scenario where underlying demands follow the basic linear relationship described in (\ref{eq:basemodel}): ${\mathbf{q}_t = \mathbf{c}_t - \mathbf{B}_t \mathbf{p}_t + \bm{\epsilon}_t}$.   Under this setting, the entries of $\mathbf{c}_t$, $\mathbf{B}_t$, and $\bm{\epsilon}_t$ are independently drawn from $N(100,20)$, $N(0,2)$, and $N(0,10)$ distributions, respectively.  Before demands are generated, $\mathbf{B}_t$ is first projected onto the set of strongly positive-definite matrices $\{\mathbf{B} :  \mathbf{B}^T + \mathbf{B} \succeq \lambda \mathbb{I} \}$ with $\lambda = 10$ as done in \S\ref{sec:results}.
We consider both the stationary case where  $\mathbf{c}_t, \mathbf{B}_t$ are fixed over time as well as the case of demand shocks, in which these underlying parameters are re-sampled from their generating  distributions at times $T/3$ and $2T/3$.   
Note that  the demands in this setting do not possess any explicit low-rank structure, nor are they governed by low-dimensional featurizations of the products.

Figure \ref{fig:highrankregrets} depicts the performance of our pricing algorithms in this linear full-rank setting, showing the average cumulative regret (over 10 repetitions with standard-deviations shaded).
 Once again, the performance of the GDG approach and our \algsecondabbrev{} algorithm with $d = N$ are essentially identical.  In this setting, the standard bandit methods slightly outperform the $\mathrm{Explo}^\mathrm{re}_\mathrm{it}$ baseline, but they do not exhibit strong performance when optimizing over a 100-dimensional action space.  Despite the lack of explicit low-rank structure in the underlying demand model, the \algsecondabbrev{} algorithm produces greater revenues than the GDG and $\mathrm{Explo}^\mathrm{re}_\mathrm{it}$ baselines for all settings of $d \in [10, 90]$ (but does fare worse than GDG if $d \ll 10$ is chosen too small).  In particular, when operating with relatively low values of $d$, the \algsecondabbrev{} method very significantly outperforms the other pricing strategies.  Similar phenomena in bandit algorithms over projected low-dimensional action subspaces have been  documented by \citetsi{Wang13, Li16, Yu17}.

\textbf{Log-linear model.}  While the linear demand model studied in this paper is one of the most popular methods for pricing products with varying elasticities, demands for products with constant elasticity are often better fit via a log-linear function of the prices \citesi{Maurice10}.  
We also evaluate the performance of our bandit methods in such a setting, where demands are determined according to the following log-linear model:
\begin{equation}
\log (\mathbf{q}_t) = \widetilde{\mathbf{c}}_t + \widetilde{\mathbf{B}}_t \log (\mathbf{p}_t + 100) + \widetilde{\bm{\epsilon}}_t 
\label{eq:loglin}
\end{equation}
In our experiment  under this setting, the entries of $\widetilde{\mathbf{c}}_t$, $\widetilde{\mathbf{B}}_t$, $\widetilde{\bm{\epsilon}}_t $ are independently drawn from  $N(5,1)$, $N(0,0.1)$, and $N(0,1)$ distributions, respectively.  Before demands are generated, $\widetilde{\mathbf{B}}_t$ is first projected onto the set of strongly positive-definite matrices $\{\mathbf{B} :  \mathbf{B}^T + \mathbf{B} \succeq \lambda \mathbb{I} \}$ with $\lambda = 0.1$. 
Again, two scenarios are considered: the stationary case where  $\widetilde{\mathbf{c}}_t, \widetilde{\mathbf{B}}_t$ are fixed over time, and the case of demand shocks, in which these underlying parameters are re-sampled from their generating  distributions at times $T/3$ and $2T/3$.   Note that this log-linear model also does not possess any explicit low-rank properties.

Figure \ref{fig:logregrets} demonstrates that the same conclusions about our algorithm's behavior in the case of full-rank linear demands also hold for this log-linear setting.   
Even though it is now quite misspecified, the \algsecondabbrev{} algorithm with a small value of $d$ performs remarkably well.   
Furthermore, the decreasing regret in Figure \ref{fig:logregrets}B illustrates how bandit pricing methods can rapidly adapt to a changing marketplace, regardless whether the underlying demands are of varying or constant elasticities.

\begin{figure*}[t] \centering
\begin{overpic}[width=0.49\textwidth]{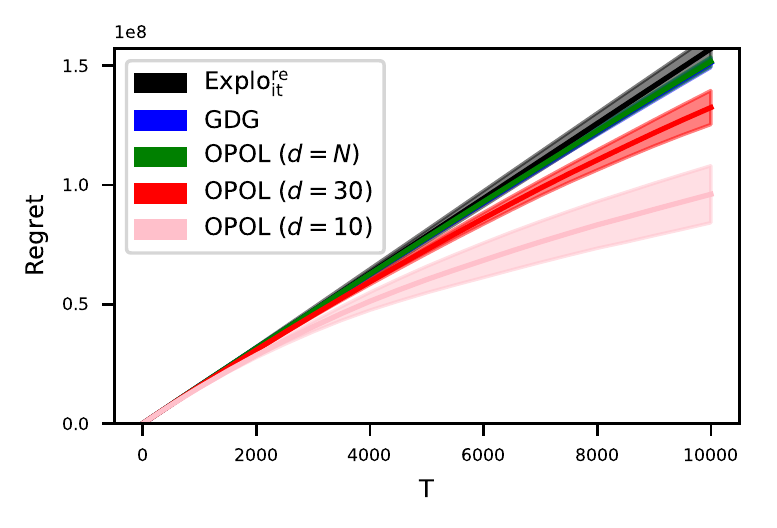}
 \put (15,68){\footnotesize \textbf{(A)} \ Model (\ref{eq:loglin}) without temporal change }
 \end{overpic} \hspace*{0mm}
\begin{overpic}[width=0.49\textwidth]{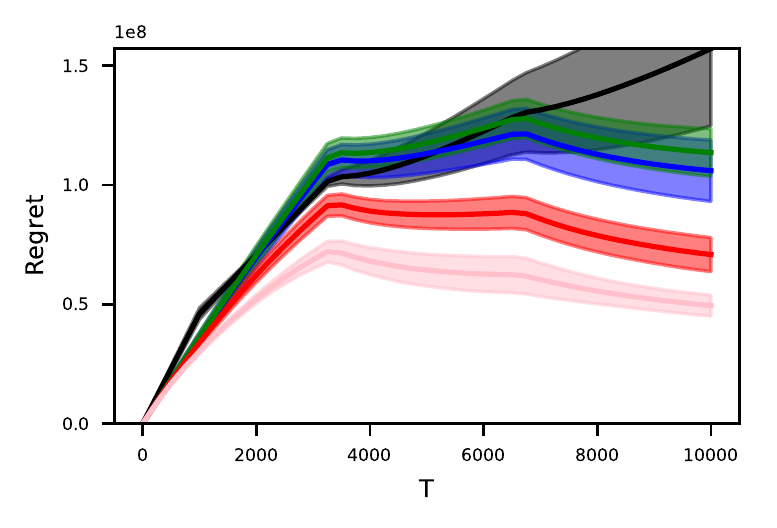} 
 \put (20, 68){\footnotesize \textbf{(B)} \ Model (\ref{eq:loglin}) with demand shocks }
\end{overpic}  \hspace*{1mm}
\vspace*{-4mm}
\caption{Average cumulative regret (over 10 repetitions with standard-deviations shaded) of various pricing strategies (for $N= 100$) when the underlying demand model is log-linear and:  \textbf{(A)}  stationary over time, \textbf{(B)} altered by structural shocks at $T/3$ and $2T/3$.   }
\label{fig:logregrets}
\end{figure*}

 \subsection{Further Details about Experiments} 
 \label{sec:expdetails}

Our simulations always set the first prices used to initialize each method, $\mathbf{p}_0$, at the center of $\mathcal{S}$.  
For each experiment in our paper, the bandit algorithm hyperparameters $\eta,\delta,\alpha$ are set as specified in Theorem \ref{thm:regretknownu}, but without knowledge of the underlying demand model (as would need to be done in practical applications).  Because the Lipschitz constant $L$ and bound $B$ are unknown in practice, these are crudely estimated prior to the initial round of our bandit pricing strategy from the observed (historical) revenues at a random collection of 100 minorly-varying prices.   To compute regret, we identify the optimal fixed price with knowledge of the underlying demand curves at each time, performing the fixed-price optimization via Sequential Least Squares Programming \citesi{Kraft88} which converges to the global optimum in our convex settings.    
In the $\mathrm{Explo}^\mathrm{re}_\mathrm{it}$ approach, transitioning from exploitation to exploration at time $T^{3/4}$ empirically outperformed the other choices we considered ($T^{1/2}, T^{2/3}, T/10, T/3$).  
Note that no matter how many experiments we run, the sensitive nature of pricing necessitates provable guarantees, which is a major strength of the adversarial regret bounds presented in this paper.

\clearpage
\section{Notation Glossary}
\label{sec:notation}

{\renewcommand{\arraystretch}{1.5} 
\begin{tabular}{l l}
$N > 0$ & Number of products to price (assumed to be large)
\\
$d > 0$ & Dimensionality of the product features (where $d \ll N$)  
\\
$t \in \{1,\dots, T\}$ & Index of each time period (i.e.\ \emph{round}) over which prices are fixed and demands aggregated
\\
$C > 0$ & A universal constant that is problem-independent and does not depend on values like $T, d, r$
\\
$\mathbf{p}_t \in \mathbb{R}^N$ &  Vector of prices for each product in period $t$ (rescaled  rather than absolute prices)
\\
$\mathbf{q}_t \in \mathbb{R}^N$ & Vector of demands for each product in period $t$ (rescaled rather than absolute demands)
\\
$R_t : \mathbb{R}^N \rightarrow \mathbb{R}$ &  Negative total revenue produced by product pricing in period $t$ (convex function)
\\
$\mathcal{S} \subset \mathbb{R}^N$ & Convex set of feasible prices (taken to be ball of radius $r$ throughout \S\ref{sec:latent}) 
\\
$\bm{\epsilon}_t \in \mathbb{R}^N$ & Random noise in observed demands of period $t$ (mean-zero with finite variance)
\\
$\bm{\epsilon}$ & Represents the full set of random demand effects  $\{\bm{\epsilon}_1, \dots, \bm{\epsilon}_T\}$
\\
$\bm{\xi}_t \in \mathbb{R}^d$ & Random noise variables drawn within each round of our bandit algorithms
\\
$\bm{\xi}$ & Represents the full set of random noise variables employed in our algorithms $\{\bm{\xi}_1, \dots, \bm{\xi}_T\}$
\\ 
$\mathbf{c}_t \in \mathbb{R}^N$ &  Vector of baseline aggregate demands for each product in period $t$ 
\\
 $\mathbf{B}_t \in \mathbb{R}^{N \times N}$ & Asymmetric positive-definite matrix of demand cross-elasticities in period $t$
 \\
$\mathbf{U} \in \mathbb{R}^{N \times d}$ & Matrix where $i$th row contains featurization of product $i$ (presumed orthogonal in \S\ref{sec:latent})
\\ 
$\widehat{\mathbf{U}} \in \mathbb{R}^{N \times d}$ & Matrix whose column-span is used to estimate the column-span of $\mathbf{U}$
\\
$\mathbf{z}_t \in \mathbb{R}^d$ & Vector which determines how product features affect the baseline demands in period $t$
\\
$\mathbf{V}_t \in  \mathbb{R}^{d \times d}$ & Asymmetric positive-definite matrix that defines changing demand cross-elasticies in period $t$
\\
$||\mathbf{x}||_2$ & Euclidean norm of vector $\mathbf{x}$
\\
$||\mathbf{A}||_\text{op}$ & Spectral norm of matrix $\mathbf{A}$ (magnitude of the largest singular value) \\
$||\mathbf{A}||_F$ & Frobenius norm of matrix $\mathbf{A}$
\\ $ \text{Unif}(\mathcal{S})$ &  Uniform distribution over set $\mathcal{S}$
\\
$ \mathbf{p}^* \in \mathbb{R}^N$ & Single best vector of prices chosen in hindsight: $\displaystyle \mathbf{p}^*
=  \argmin_{\mathbf{p} \in \mathcal{S}} \mathbb{E} \sum_{t=1}^T R_t(\mathbf{p})$ 
\\
$f_t : \mathbb{R}^d \rightarrow \mathbb{R}$ & Function such that $f_t(\mathbf{x}) = \mathbb{E}_{\bm{\epsilon}}[R_t(\mathbf{p})]$ for $\mathbf{x} = \mathbf{U}^T \mathbf{p}$
\\
$f_{t,\widehat{\mathbf{U}}} : \mathbb{R}^d \rightarrow \mathbb{R}$ & Function such that $f_{t,\widehat{\mathbf{U}}}(\mathbf{x}) = \mathbb{E}_{\bm{\epsilon}}[R_t(\mathbf{p})]$ for $\mathbf{x} = \widehat{\mathbf{U}}^T \mathbf{p}$
\\
$ \eta, \delta, \alpha > 0$ & User specified hyperparameters of our bandit pricing algorithms
\\
$\sigma^2 > 0$  & Sub-Gaussian parameter that specifies magnitude of noise effects in the observed demands
\\
$\mathbf{U}^T(\mathcal{S})$ & $d$-dimensional actions that correspond to feasible prices: $\big\{ \mathbf{x} \in \mathbb{R}^d : \mathbf{x} = \mathbf{U}^T \mathbf{p} \ \text{for some }  \mathbf{p}  \in \mathcal{S} \big\}$
\\ 
 $r_{\smallup}, r_{\smalldown} > 0$  & Radius of Euclidean balls containing/contained-within $\mathbf{U}^T(\mathcal{S})$, with $r_{\smallup} \ge r_{\smalldown}$
\\
$B > 0$ & Upper bounds the magnitude of $\mathbb{E} [R_t(\mathbf{p})]$ over all $\mathbf{p} \in \mathcal{S}$, $t = 1,\dots,T$ 
\\
$L > 0$ & Lipschitz constant of each $f_t (\mathbf{x})$ over all $\mathbf{x} \in \mathbf{U}^T(\mathcal{S})$, $t = 1,\dots,T$
\\
$b > 0$ & Upper bounds the magnitude of $\mathbf{z}_t, \mathbf{V}_t$ for $t = 1,\dots,T$ ($||\mathbf{z}_t||_2 \le b$ and $||\mathbf{V}_t||_\text{op} \le b$)
\\
$r \ge 1$ & Radius of Euclidean ball adopted as the feasible set of (rescaled) prices throughout \S\ref{sec:latent}
\end{tabular}
} 

\clearpage
\nocitesi{Flaxman05}
\bibliographystylesi{plainnat}
\bibliographysi{banditPricing}

\end{document}